\documentclass[11pt]{article}

\usepackage[final]{acl}

\usepackage{times}
\usepackage{latexsym}

\usepackage[T1]{fontenc}

\usepackage[utf8]{inputenc}

\usepackage{microtype}

\usepackage{inconsolata}

\usepackage{bm}
\usepackage{amsmath, amssymb, amsfonts, amsthm}
\usepackage{graphicx}
\usepackage{booktabs}
\usepackage{multirow}
\usepackage{amsfonts}
\usepackage{color,xcolor,colortbl}
\usepackage{adjustbox}
\usepackage{tabularx}
\usepackage{subcaption}
\usepackage{cleveref}
\usepackage{makecell}
\usepackage{enumitem}
\usepackage{mathtools}
\usepackage{hyperref}
\usepackage{paralist}
\usepackage{thmtools, thm-restate}
\usepackage{array}
\usepackage[symbol]{footmisc}

\declaretheorem{lemma, proposition, corollary}[
    style=plain
]
\newcommand{\headercolor}{\rowcolor{blue!10}}
\newcommand{\headercolorgray}{\rowcolor{gray!10}}
\definecolor{darkgreen}{RGB}{0,130,0}
\definecolor{darkred}{RGB}{160,0,0}

\title{Towards Reliable Truth-Aligned Uncertainty Estimation in\\Large Language Models}

\author{
  Ponhvoan Srey$^1$ \enskip Quang Minh Nguyen$^2$ \enskip Xiaobao Wu$^{3}$\thanks{Corresponding Authors.} \enskip Anh Tuan Luu$^{1,4}$\footnotemark[1]\\
  $^1$Nanyang Technological University \quad $^2$KAIST \\ $^3$Shanghai Jiao Tong University \quad $^4$VinUniversity \\
  \texttt{\{ponhvoan002, anhtuan.luu\}@ntu.edu.sg}\\
  \texttt{qm.nguyen@kaist.ac.kr} \quad  \texttt{xiaobaowu@sjtu.edu.cn}
  }

\begin{document}
\maketitle
\begin{abstract}
Uncertainty estimation (UE) aims to quantify uncertainty in large language model (LLM) outputs to flag potential hallucinations and support reliable deployment.
However, UE metrics often exhibit unstable performance across models and datasets, limiting their practical utility. %
We formalise this phenomenon as \emph{proxy failure}, arising because UE metrics are derived from model behaviour rather than being explicitly grounded in factual correctness.
We show that when mutual information between a UE score and correctness is low, the score's discriminability must be close to chance.
To address this, we propose \emph{Truth AnChoring} (TAC), a lightweight post-hoc calibration method that maps raw UE scores to truth-aligned correctness probabilities.
Even with noisy and few-shot supervision, TAC continues to improve calibration, providing a practical calibration protocol.
Our findings highlight the limitations of treating heuristic UE metrics as direct indicators of uncertainty about factual correctness, and position our TAC as a principled step toward more reliable uncertainty estimation for LLMs.\footnote[2]{The code repository is available at \url{https://github.com/ponhvoan/TruthAnchor/}.}

\end{abstract}

\section{Introduction}

Large Language Models (LLMs) have advanced rapidly in capability \citep{achiam2023gpt, guo2025deepseek, yang2025qwen3} and are being deployed across a broad range of applications,
including in high-risk domains, such as the medical \citep{nori2023capabilities, thirunavukarasu2023large, goh2024large} and financial fields \citep{wu2023bloomberggpt, lopez2023can}.
Yet, they still generate falsehoods or hallucinate with high fluency and apparent confidence, making errors difficult to detect from the generations alone \citep{zhang2025siren}.
This motivates uncertainty estimation (UE).
UE aims to assign an uncertainty score to each response, and flag potential hallucinations.
This enables appropriate and timely interventions on the response, like abstention or verification \citep{liu2025uncertainty}.

However, recent studies show that UE metrics do not behave stably across configurations, with scores that appear effective in one setting degrading sharply in another \citep{vashurin-etal-2025-benchmarking, wang2025measuring, shorinwa2025survey, tan-etal-2025-consistent, tomov2025illusion}. We argue that this is because these UE metrics are not directly measuring correctness.
For example, predictive entropy quantifies uncertainty of the next-token distribution, and therefore, reflects local continuation stability rather than semantic truth.
In the same manner, many other UE metrics function as heuristics or \emph{proxies} for correctness as they may correlate with truth in some contexts, but that relationship is fragile.

In this work, we present a formal study of this phenomenon, which we call \textit{\textbf{proxy failure}}.
Our contributions are three-fold:
\begin{inparaenum}[(i)]
    \item We demonstrate explicitly the non-discriminability of UE scores under low information regimes, and highlight how this can manifest in entropy as a concrete example.
    \item We can alleviate this issue through \emph{Truth AnChoring (TAC)}. By connecting uncertainty scores to correctness labels, they can be post-hoc calibrated to align with truth, yielding significant reductions in the expected calibration error (ECE) (\Cref{fig:reliability_diagrams}).
    \item Finally, we provide a practical route for when gold annotations are scarce. Even weak truth signals, such as few-shot and noisy supervision, are still useful for learning truth-aligned uncertainty estimates.
\end{inparaenum}
More broadly, our results caution that reliable UE cannot rest on heuristics alone. We position TAC as a crucial calibration layer over existing UE scores to inject an explicit notion of correctness to align them with factuality.
With these contributions, we hope to steer the field toward more principled post-hoc calibration.

\begin{figure*}[ht]
    \centering
    \setlength{\tabcolsep}{3pt}
    \renewcommand{\arraystretch}{1.2}
    
    \begin{tabular}{>{\centering\arraybackslash}m{0.8cm}
                    >{\centering\arraybackslash}m{0.16\textwidth}
                    >{\centering\arraybackslash}m{0.16\textwidth}
                    >{\centering\arraybackslash}m{0.16\textwidth}
                    >{\centering\arraybackslash}m{0.16\textwidth}
                    >{\centering\arraybackslash}m{0.16\textwidth}}
        
        \scriptsize \textbf{Vanilla} 
        & \includegraphics[width=\linewidth]{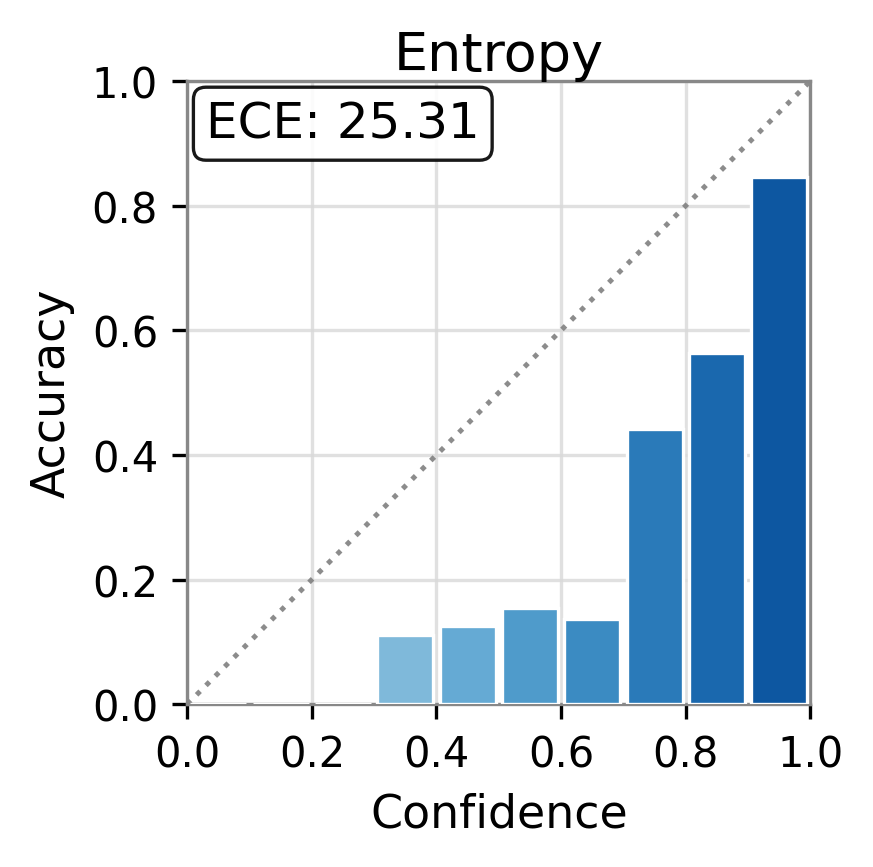}
        & \includegraphics[width=\linewidth]{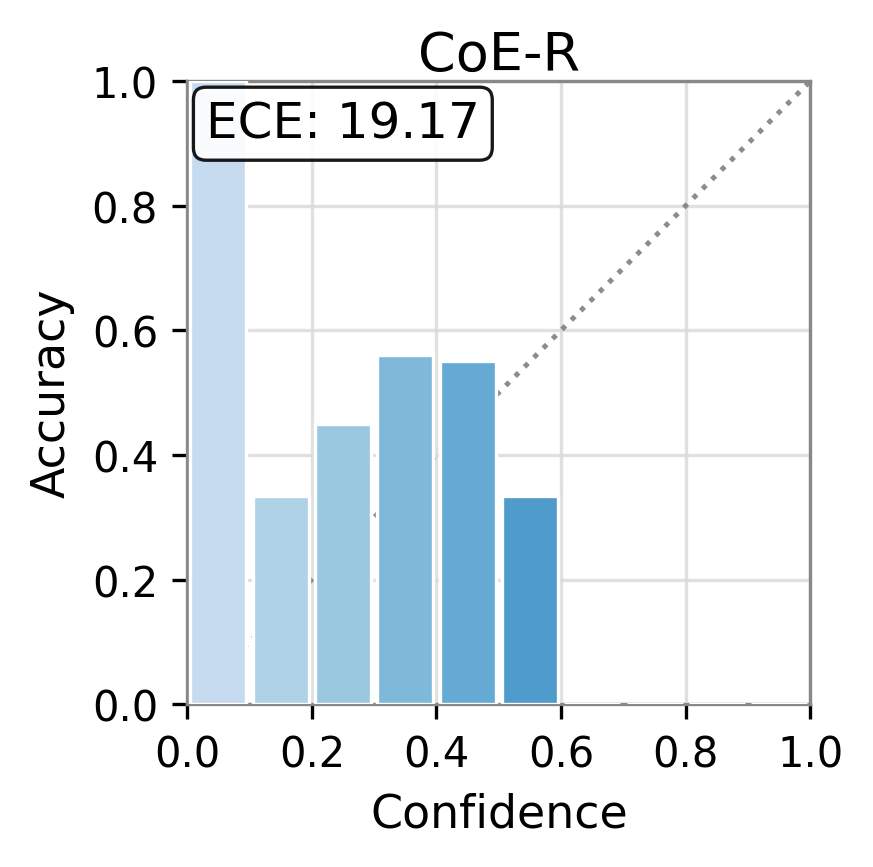}
        & \includegraphics[width=\linewidth]{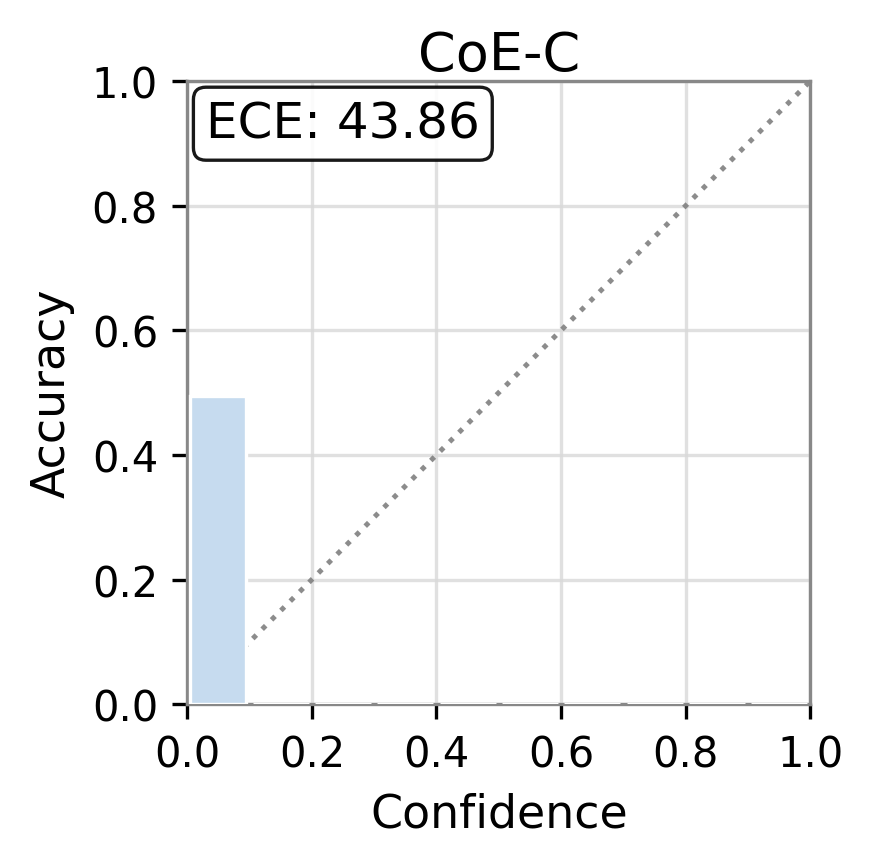}
        & \includegraphics[width=\linewidth]{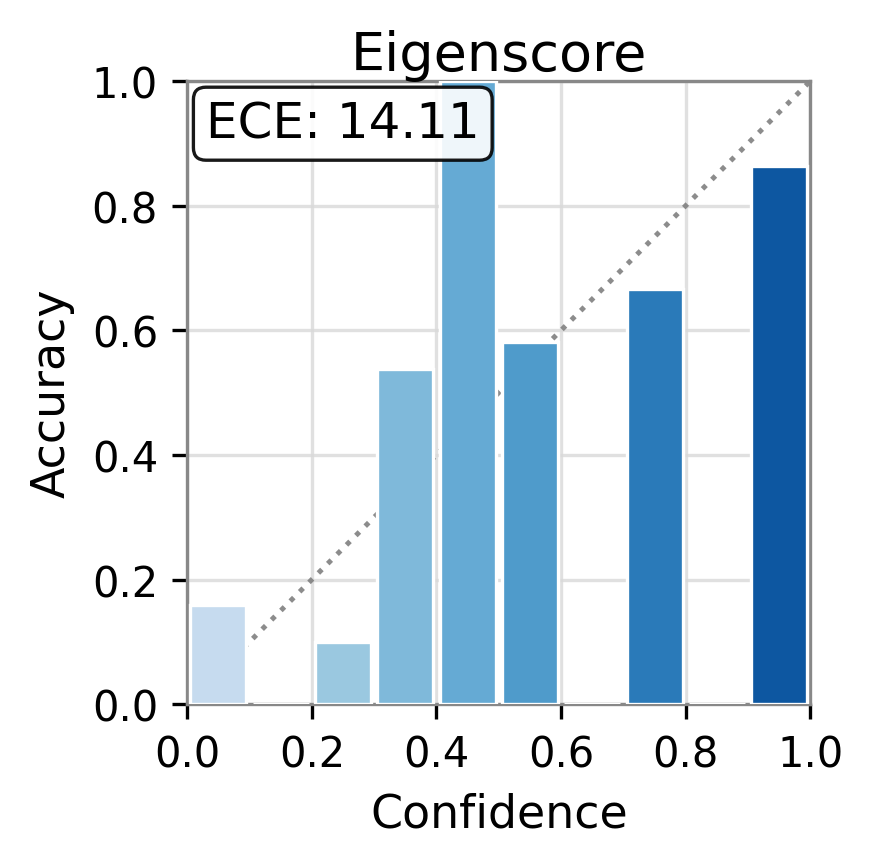}
        & \includegraphics[width=\linewidth]{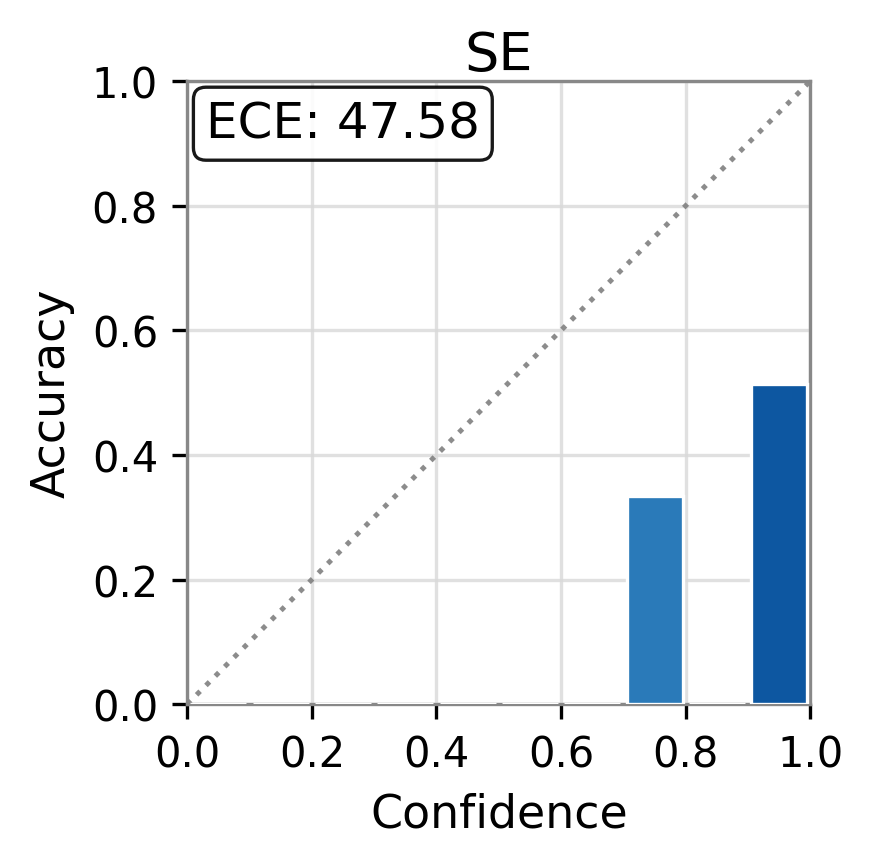}\\
        \scriptsize \textbf{TAC} 
        & \includegraphics[width=\linewidth]{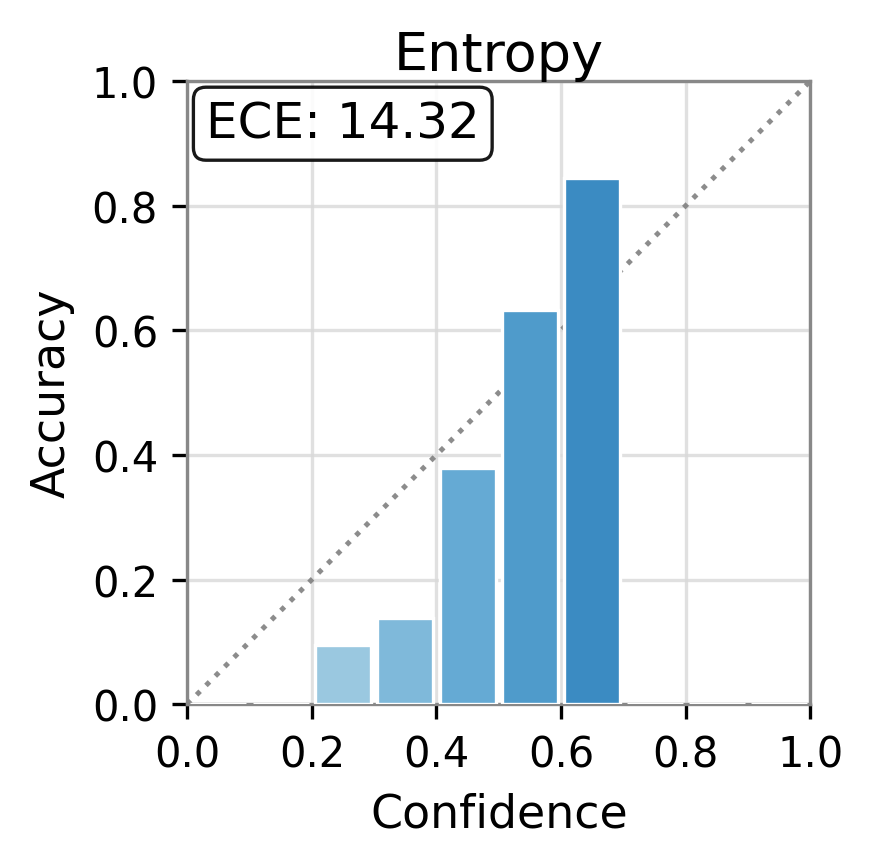}
        & \includegraphics[width=\linewidth]{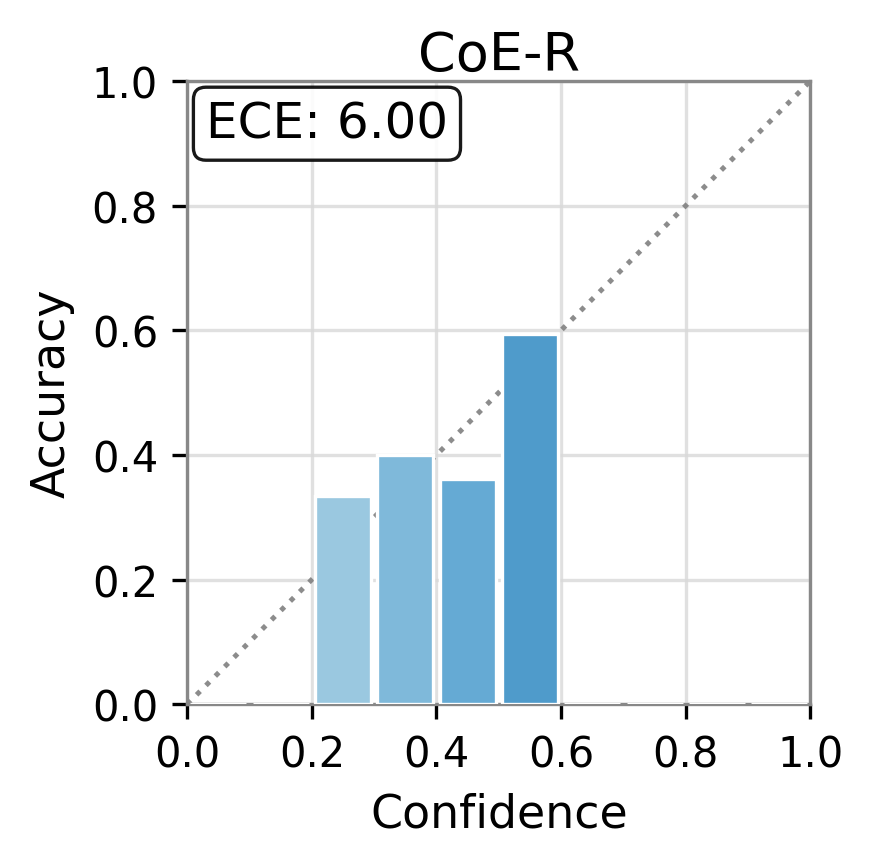}
        & \includegraphics[width=\linewidth]{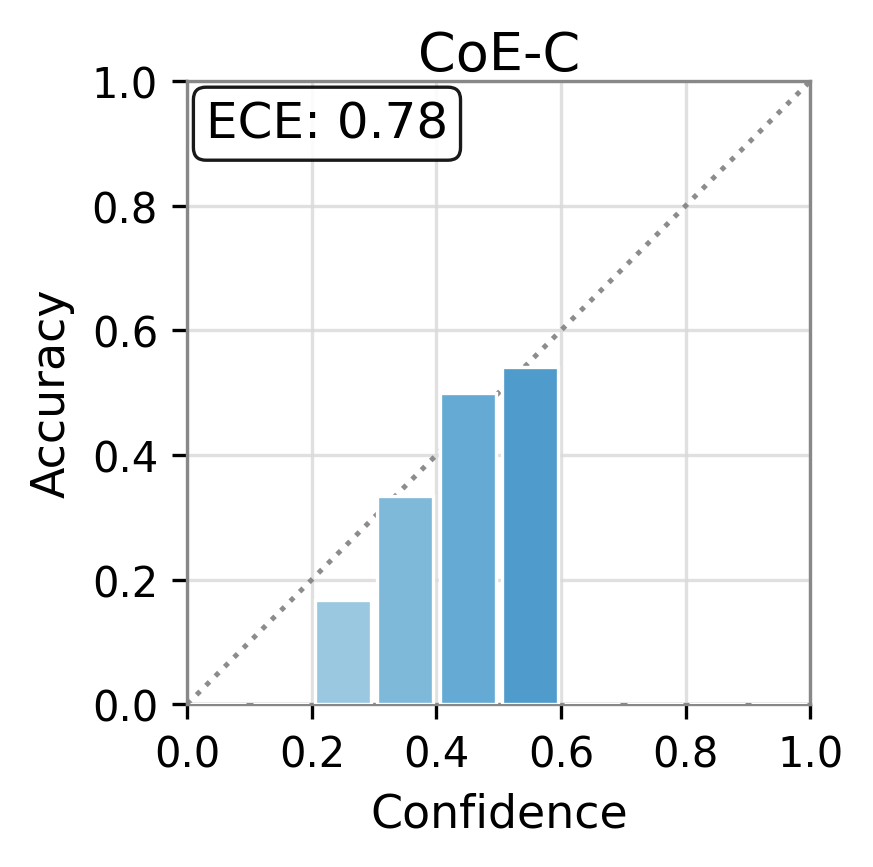}
        & \includegraphics[width=\linewidth]{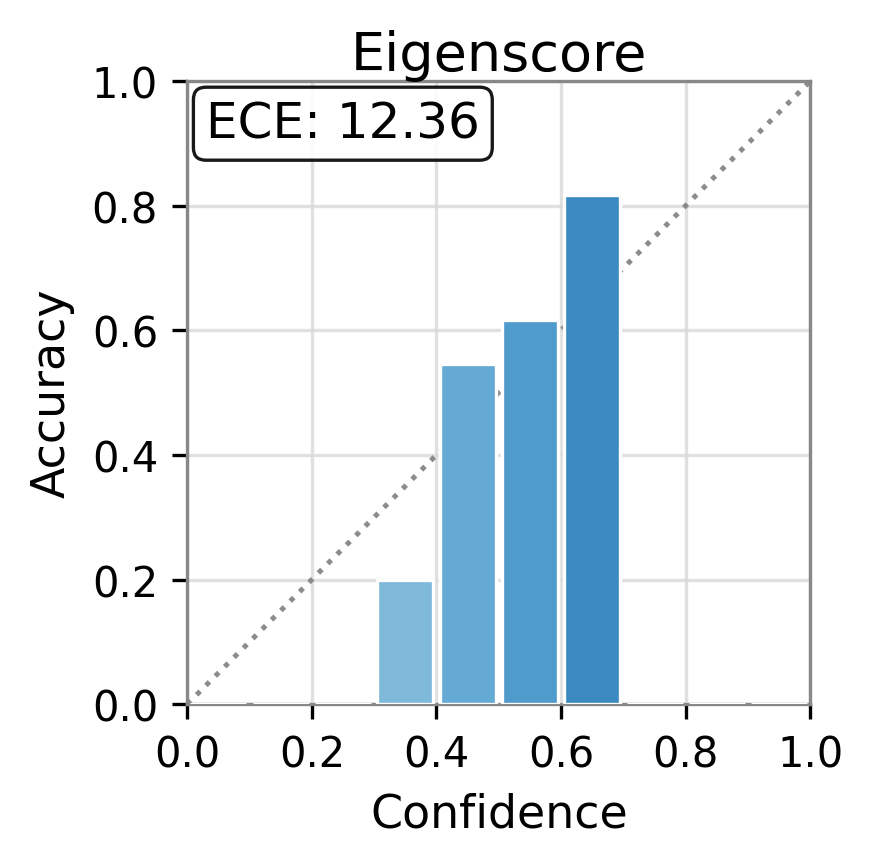}
        & \includegraphics[width=\linewidth]{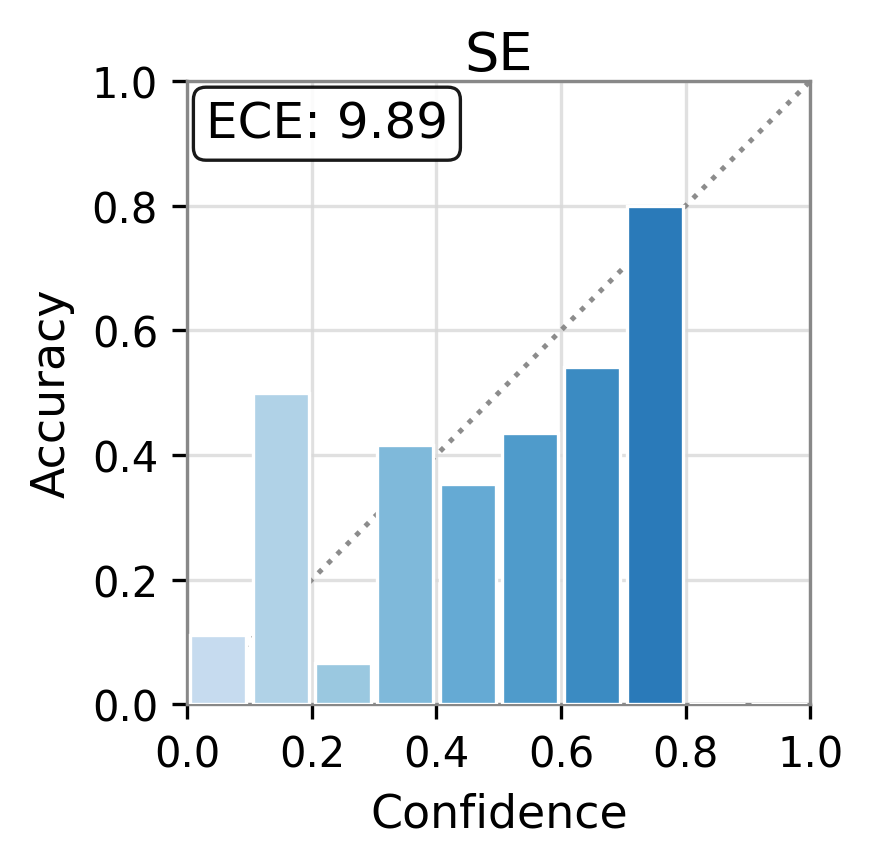}\\
    \end{tabular}
    \caption{Reliability diagrams of widely used and recent uncertainty methods, and our Truth-AnChored (TAC) scores.}
    \label{fig:reliability_diagrams}
\end{figure*}

\section{Related Work}

Existing UE methods typically derive uncertainty scores from the LLM's own behaviour.
Popular approaches examine the next-token probability distribution, and compute the uncertainty of this underlying distribution via
measures such as entropy \citep{huang2023look}, perplexity \citep{si2022prompting}, or their variants \citep{malinin2020uncertainty, shih2023long}.
Others rely on response consistency across multiple samples, both at the output token level \citep{manakul2023selfcheckgpt, kuhn2023semantic, lin2023generating} and representation level \citep{chen2024inside, li2025semantic}, with the assumption that if the model emits equivalent responses, then it has sufficient parametric knowledge and has generated a correct answer.
Others avoid sampling altogether due to its computational requirements, and extract signals from internal representations and formulate a measure to track uncertainty \citep{wang2024latent, sriramanan2024llm, yin2024characterizing}.
Evidently, the prevailing paradigm is to formulate a score that should, in principle, track factuality.
Unfortunately, as noted in recent works \citep{vashurin-etal-2025-benchmarking, shorinwa2025survey},
performance of such metrics varies across datasets.
Our work reveals these metrics can only serve as imperfect proxies for truth, and can fail to meaningfully represent uncertainty in general.

Closely related to our work are probe-based UE methods that train a lightweight classifier on LLM internal states to predict hallucination risk, resulting in strong uncertainty estimates \citep{azaria-mitchell-2023-internal, burns2022discovering, ji-etal-2024-llm, srey2025unsupervised}. 
Most efforts are focused on curating more informative features, such as Mahalanobis distance of embeddings \citep{vazhentsev-etal-2025-token}, lookback ratios over context \citep{chuang2024lookback}, and attention activation maps \citep{he-etal-2024-llm}.
Our work contributes to and differs from this suite of methods in two ways. First, we establish that probe-based UE is principled and well-grounded, thus explaining their superior empirical performance.
Second, we emphasise a minimal and previously overlooked instantiation: post-hoc calibration of the raw uncertainty score itself.
Most pertinent to our work is CUE \cite{li-etal-2025-towards}, which learns a corrector to adjust uncertainty scores using correctness supervision with additional BERT-based embedding features. Our work is complementary but distinct. We provide a formal account of \emph{why} such correction is needed, and empirically demonstrate even the scalar proxy alone can be turned to a truth-aligned estimator.

\begin{figure*}
    \centering
    \begin{subfigure}{0.32\textwidth}
    \includegraphics[width=\linewidth]{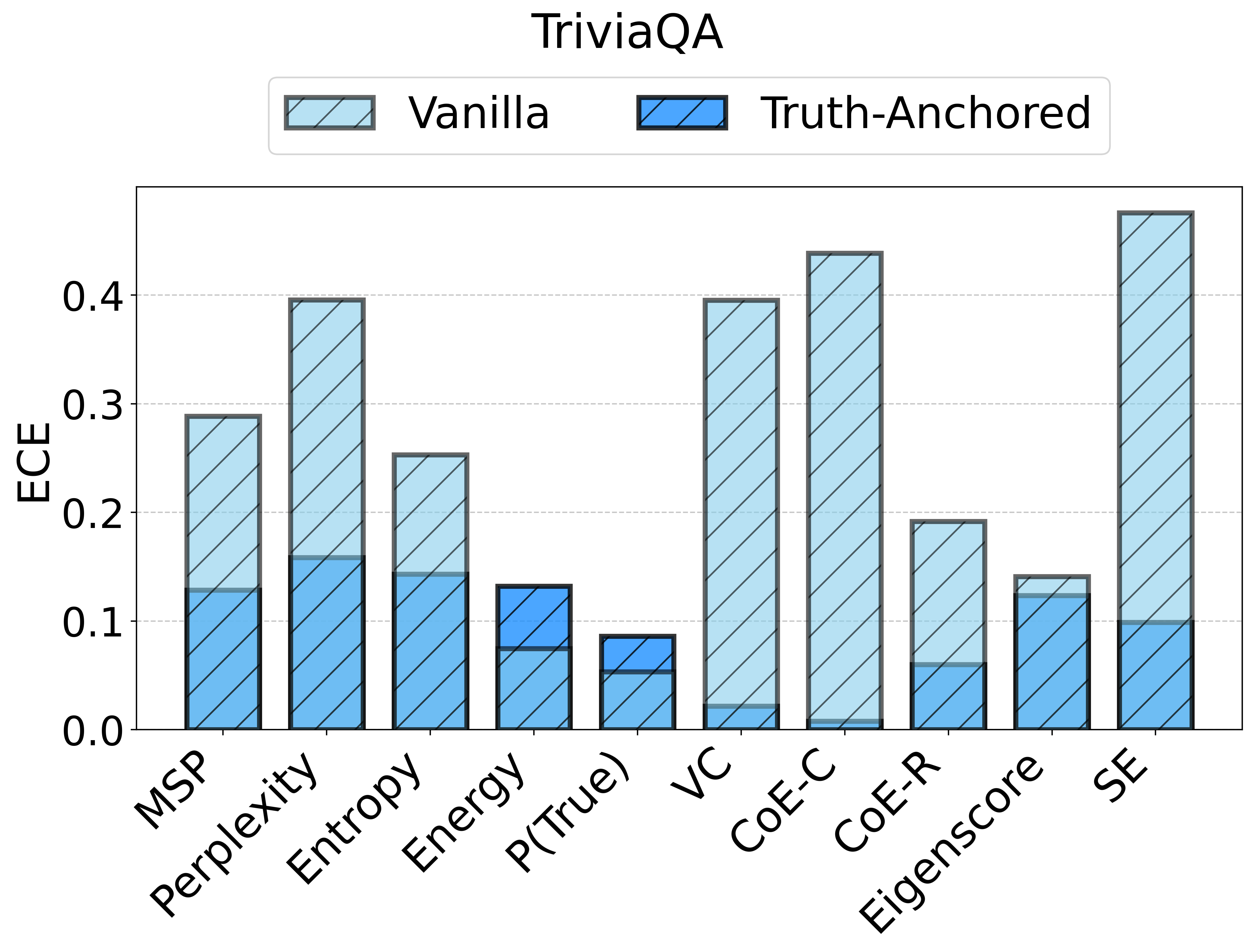}
    \end{subfigure}
    ~
    \begin{subfigure}{0.32\textwidth}
    \includegraphics[width=\linewidth]{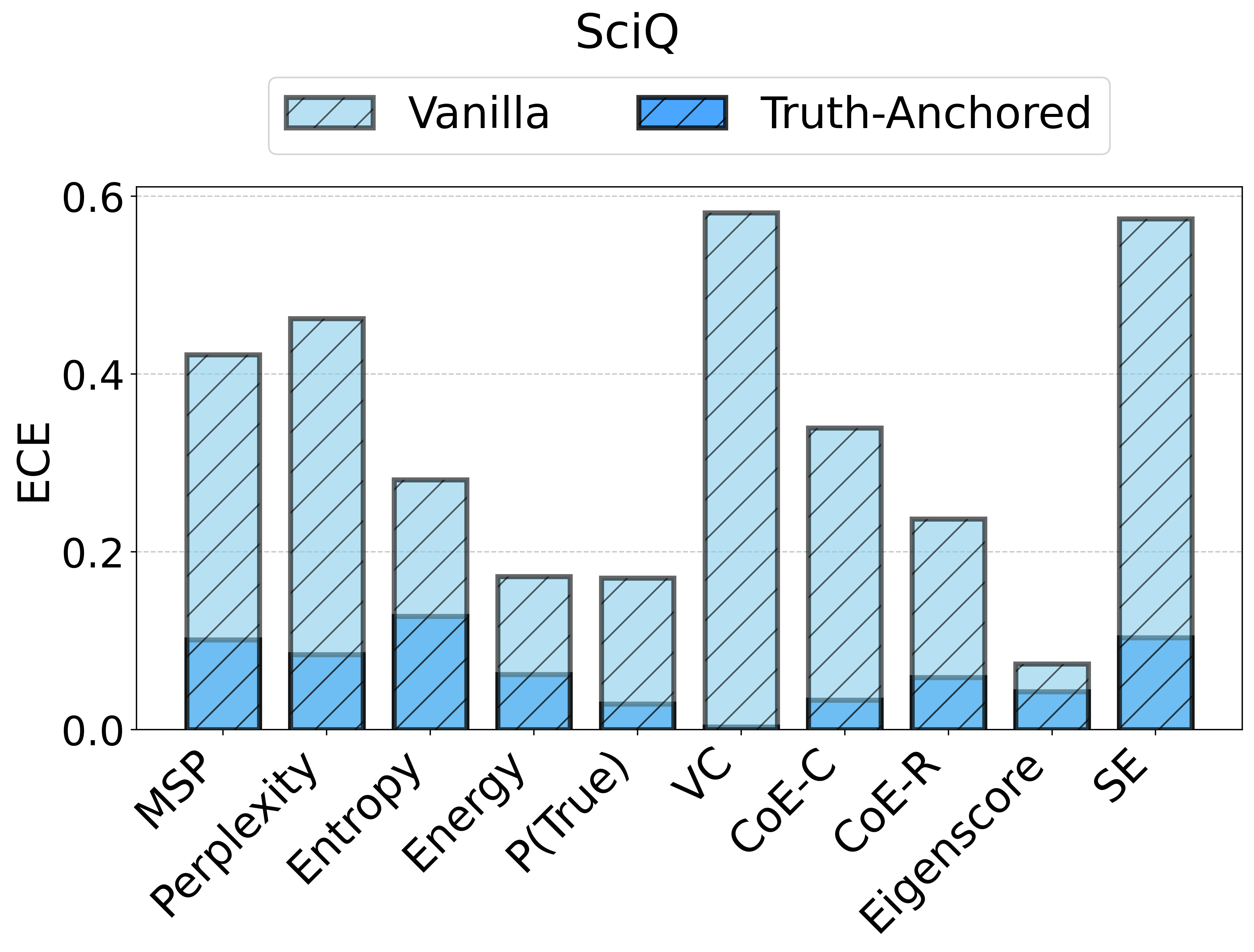}
    \end{subfigure}
    ~
    \begin{subfigure}{0.32\textwidth}
    \includegraphics[width=\linewidth]{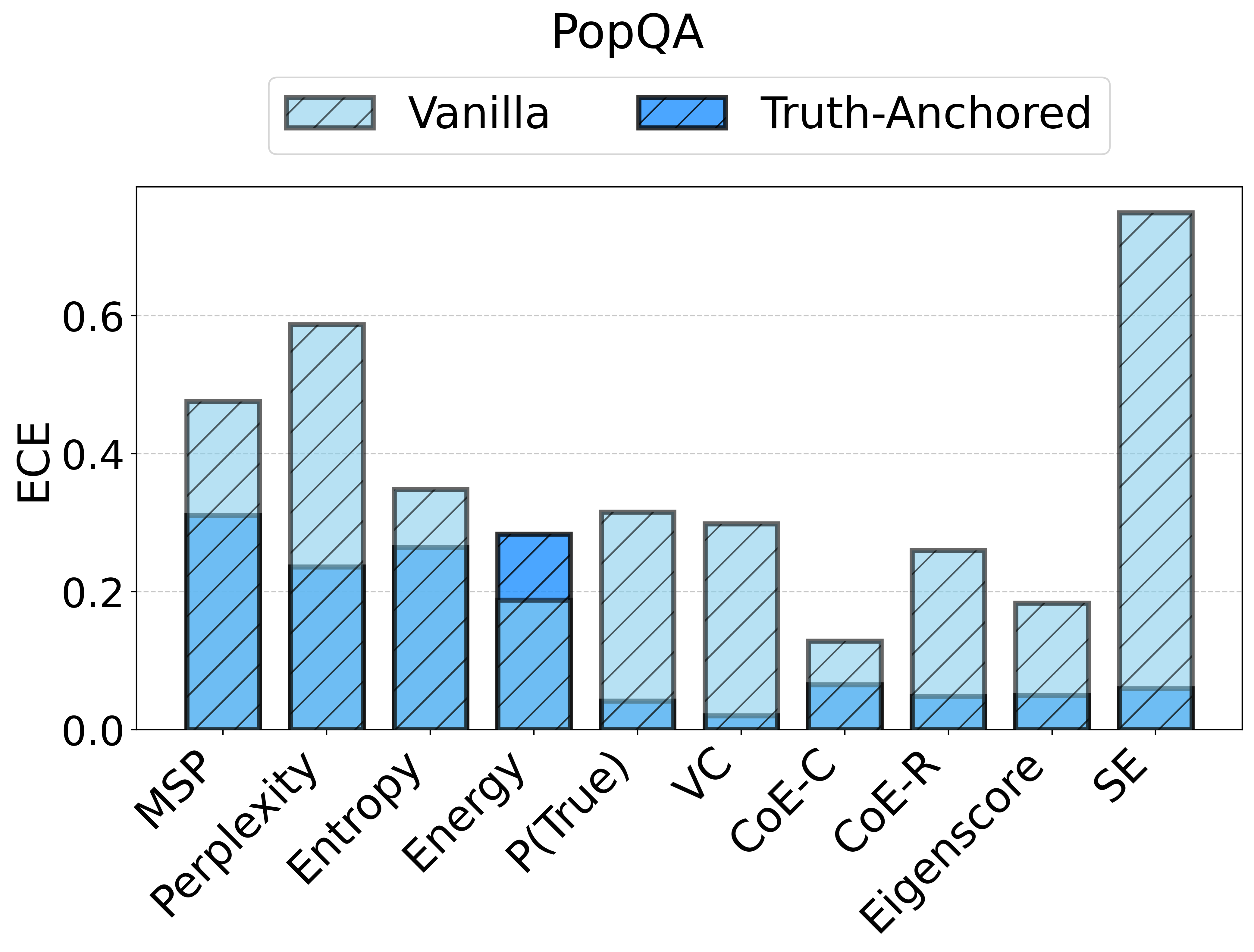}
    \end{subfigure}

    \begin{subfigure}{0.32\textwidth}
    \includegraphics[width=\linewidth]{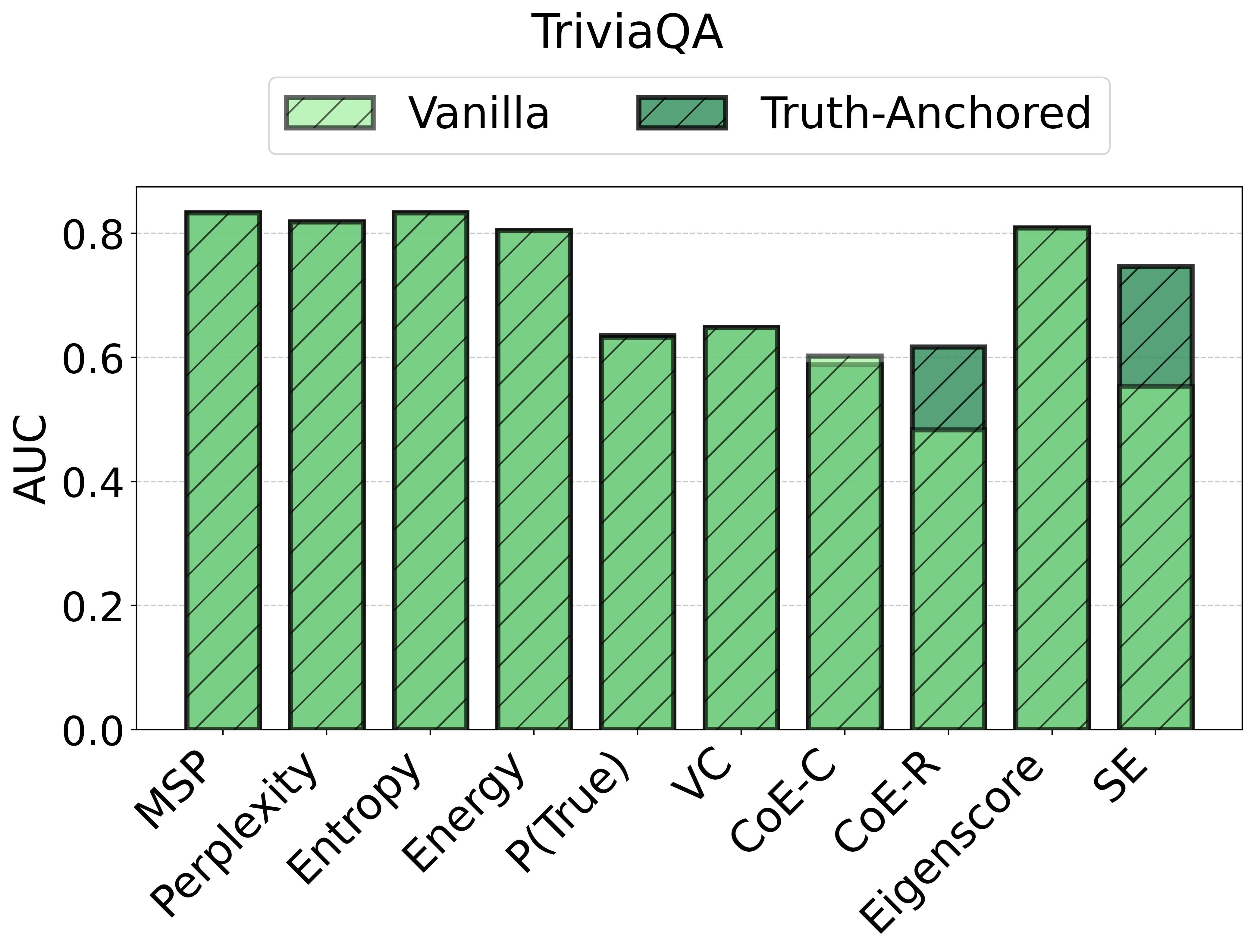}
    \end{subfigure}
    ~
    \begin{subfigure}{0.32\textwidth}
    \includegraphics[width=\linewidth]{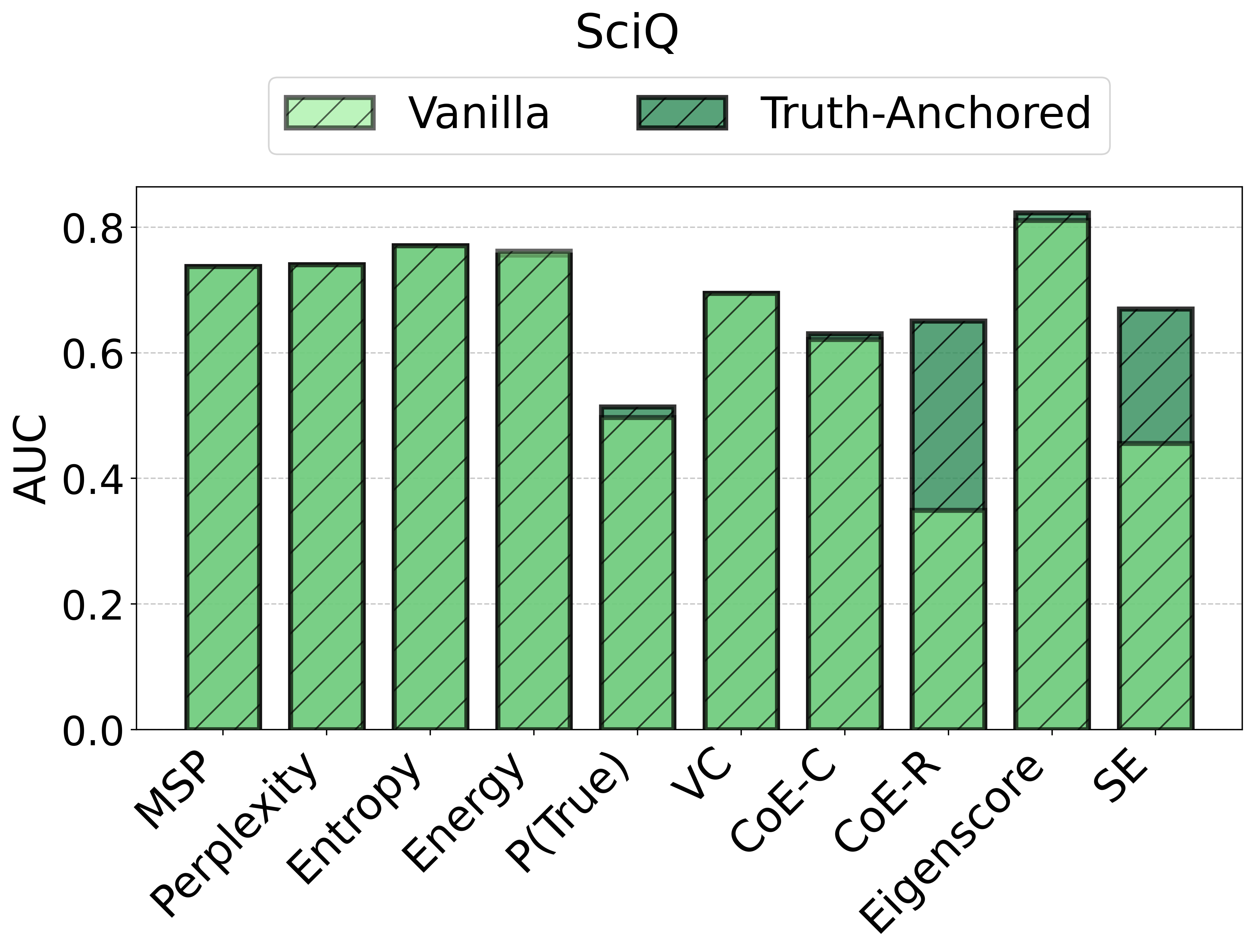}
    \end{subfigure}
    ~
    \begin{subfigure}{0.32\textwidth}
    \includegraphics[width=\linewidth]{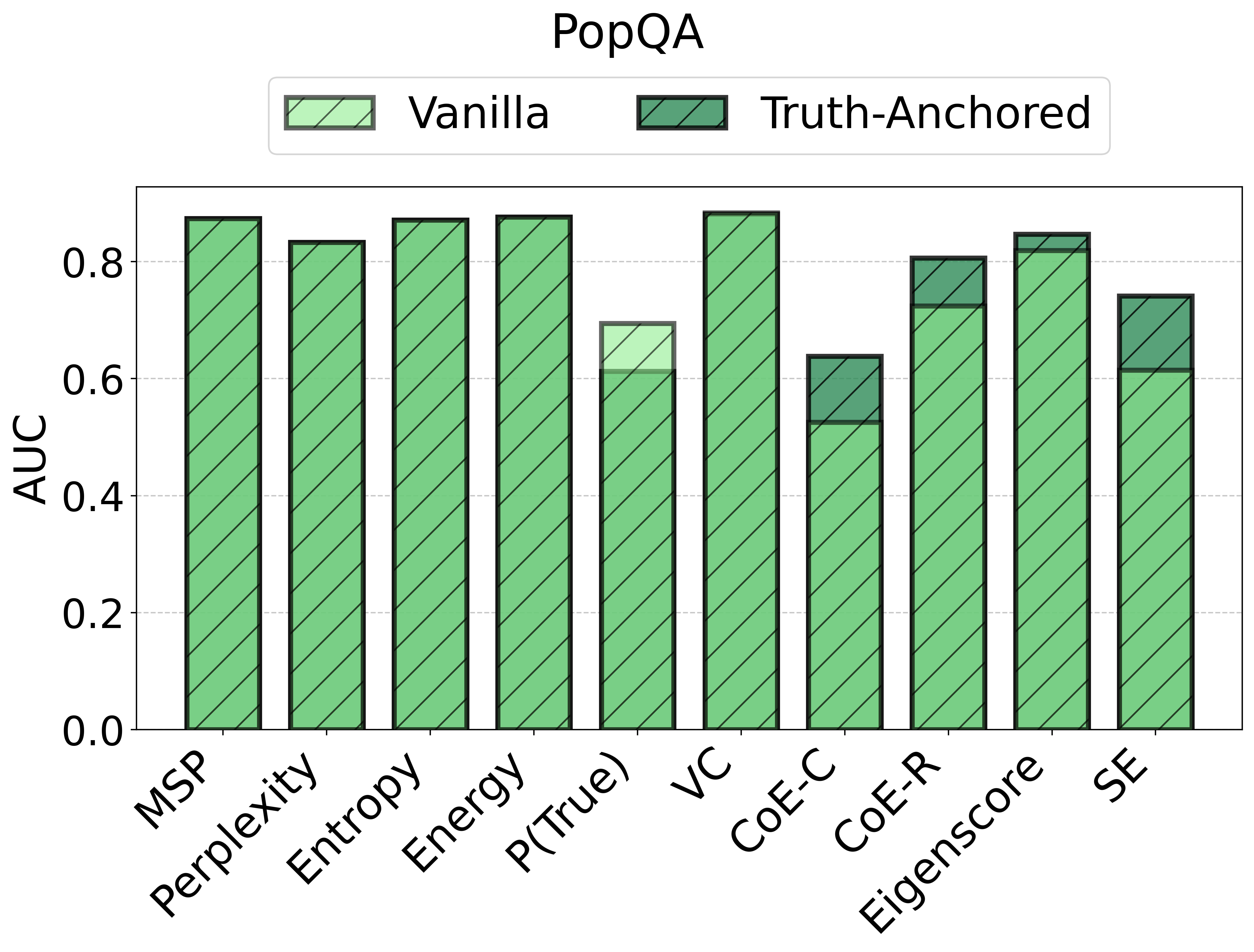}
    \end{subfigure}
    
    \caption{Our proposed TAC significantly reduces ECE ($\downarrow$) and frequently increases AUC ($\uparrow$), improving calibration and discrimination.}
    \label{fig:ece_comparison}
\end{figure*}

\section{Proxy Failure of Uncertainty Scores}
\label{sec:study}

We now formulate the proxy failure phenomenon studied in this work. Our goal is to answer the following question: \emph{When does a proxy uncertainty score fail to meaningfully discriminate between correct and incorrect generations?}
With \Cref{prop:miAUC}, we establish that this arises whenever the uncertainty score carries little information about correctness.

\paragraph{Notation.}
Given a query $X=x$ drawn from $\mathcal{X}$, a language model induces a conditional distribution $P(\cdot \mid x)$ over responses $Y \in \mathcal{Y}$. We assume a fixed truth function $g^\star:\mathcal{X}\times \mathcal{Y} \to \{0,1\}$, and define correctness label $C=g^\star(X,Y)$, with $C=1$ indicating correct response. We write $p_c=P(C=1)$ for the correctness prior.
Here, we consider the confidence score $S(X,Y)$, which we simply take as the inverse of uncertainty.
For any real-valued score $S$, define the population Area Under the receiver-operating Characteristic curve (AUC) \citep{davis2006relationship} as
\begin{equation}
    \mathrm{AUC}(S) = P(S_+>S_-)
\end{equation}
where $S_+ \sim P_+ = P(S\mid C=1)$ and $S_- \sim P_-= P(S \mid C=0)$. AUC is our primary measure of discriminability, where $AUC(S)=1/2$ corresponds to chance ranking.

\subsection{Mutual Information Controls Discriminability.}

\Cref{prop:miAUC} formalises the intuition that \emph{if an uncertainty score carries little information about correctness, then its discriminative ability must be close to chance.}

\begin{restatable}{proposition}{miAUC}
\label{prop:miAUC}
Let $S$ be any real-valued score with continuous class-conditional distributions, and $p_c = P (C=1) \in (0,1)$. Then, 
\begin{equation}
    \left| \mathrm{AUC(S)} - \frac{1}{2} \right| \; \leq \; \sqrt{\frac{I(C,S)}{2p_c(1-p_c)}},
\end{equation}
where $I$ denotes mutual information.
In particular, if $I(C,S)$ is small, then the discriminability of $S$ must be close to chance.

\end{restatable}
\noindent We defer proofs to Appendix~\ref{app:proofs}. Next, we show how this can manifest in entropy as a concrete example.

\subsection{Non-Discriminability of Entropy}

Predictive entropy illustrates proxy failure clearly.
A model may commit early to an incorrect concept, yet continue to answer fluently with low local uncertainty. A correct continuation can exhibit an almost identical entropy pattern. Entropy, thus, is sensitive to local continuation structure rather than semantic truth, a global characteristic of the entire sequence.
The consequence is that entropy can carry vanishing information about correctness and can therefore, by \Cref{prop:miAUC}, fail to discriminate.
\Cref{prop:peFail} formalises this occurrence.

\begin{restatable}{proposition}{peFail}
\label{prop:peFail}

Fix a truth function $g^\star$ and any predictive entropy proxy $S_{pe}$.
For any $\varepsilon>0$, there exists a conditional distribution $P(\cdot \mid x)$ such that
\begin{equation}
    I(C, S_{pe}) \le \varepsilon
\end{equation}

Consequently, by \Cref{prop:miAUC},
\begin{equation}
    \left| \mathrm{AUC}(S_{pe}) - \frac{1}{2} \right| \; \leq \; \sqrt{\frac{\varepsilon}{2p_c(1-p_c)}}
\end{equation}

In particular, the AUC of predictive entropy can be made arbitrarily close to chance.
    
\end{restatable}

\begin{proof}[Proof Idea]
\renewcommand{\qedsymbol}{}
We provide a proof by construction. For each query, we consider a setting in which the model selects between one correct and one incorrect response that diverge only at a key semantic commitment, while the remaining tokens are routine continuations. This makes the entropy profiles indistinguishable despite different correctness labels.
We construct our required distribution $P$ to be a mixture of this setting and an arbitrary LLM.
Under $P$, we show that mutual information between predictive entropy and correctness vanishes to zero.
    
\end{proof}

\section{Truth Anchoring as a Remedy}

We have demonstrated that proxy metrics can fail as they may only reflect model behaviour but carry little information about global semantic correctness.
This suggests that reliable uncertainty estimation cannot rely on proxies alone.
Instead, we argue that uncertainty estimation should explicitly target, and therefore, be \emph{anchored} to, a notion of truth or correctness.

Given an information variable $\mathcal{I}$ available at inference time, the natural target is the posterior probability of correctness $p^\star(\mathcal{I}) = P(C=1 \mid \mathcal{I})$.
Under a strictly proper scoring rule $\ell$, the \emph{truth-anchored} posterior $p^\star$ is exactly the Bayes optimal predictor. That is, among all predictors $q(\mathcal{I})\in[0,1]$, it uniquely minimises the population risk $\mathcal{R}(q)=\mathbb{E}\left[\ell(q(\mathcal{I}),C)\right]$ \citep{gneiting2007strictly}.
Thus, rather than treating a raw proxy score itself as uncertainty, we aim to learn this posterior probability that is optimally aligned with truth under a chosen scoring rule.
We call the learning of this mapping from raw scalar to truth-aligned score \textit{\textbf{Truth AnChoring (TAC)}}.
This is consistent with converging evidence that probe-based methods showcase superior performance.

\begin{table*}[!t]
\centering
\renewcommand{\arraystretch}{1.25}
\setlength{\tabcolsep}{2pt}
\resizebox{\textwidth}{!}{
\begin{tabular}{l ccc|ccc|ccc|ccc|ccc|ccc}
\toprule
& \multicolumn{6}{c}{\textbf{TriviaQA}} & \multicolumn{6}{c}{\textbf{SciQ}} & \multicolumn{6}{c}{\textbf{PopQA}} \\
\cmidrule (lr){2-7}\cmidrule (lr){8-13} \cmidrule (lr){14-19}
\textbf{Method} & \multicolumn{3}{c}{\textbf{ECE ($\downarrow$)}} & \multicolumn{3}{c}{\textbf{AUC ($\uparrow$)}} & \multicolumn{3}{c}{\textbf{ECE ($\downarrow$)}} & \multicolumn{3}{c}{\textbf{AUC ($\uparrow$)}} & \multicolumn{3}{c}{\textbf{ECE ($\downarrow$)}} & \multicolumn{3}{c}{\textbf{AUC ($\uparrow$)}}\\
\cmidrule (lr){2-4}\cmidrule (lr){5-7}\cmidrule (lr){8-10}\cmidrule (lr){11-13} \cmidrule (lr){14-16}\cmidrule (lr){17-19}
& \textbf{Vanilla} & \textbf{TAC} & \textbf{$\Delta$}& \textbf{Vanilla} & \textbf{TAC} & \textbf{$\Delta$} & \textbf{Vanilla} & \textbf{TAC} & \textbf{$\Delta$}& \textbf{Vanilla} & \textbf{TAC} & \textbf{$\Delta$} & \textbf{Vanilla} & \textbf{TAC} & \textbf{$\Delta$}& \textbf{Vanilla} & \textbf{TAC} & \textbf{$\Delta$} \\

\midrule
\headercolor
\multicolumn{19}{c}{\textbf{Qwen-3-4B}} \\
MSP & 63.68 & 5.52 & \textcolor{darkgreen}{-58.16} & 73.82 & 73.58 & \textcolor{darkred}{-0.24} & 55.97 & 3.89 & \textcolor{darkgreen}{-52.08} & 70.07 & 70.02 & \textcolor{darkred}{-0.05} & 80.21 & 6.07 & \textcolor{darkgreen}{-74.13} & 33.02 & 69.14 & \textcolor{darkgreen}{36.12} \\
Perplexity & 62.71 & 6.12 & \textcolor{darkgreen}{-56.59} & 73.60 & 73.48 & \textcolor{darkred}{-0.12} & 51.53 & 9.95 & \textcolor{darkgreen}{-41.58} & 69.98 & 69.98 & 0.00 & 80.03 & 4.65 & \textcolor{darkgreen}{-75.38} & 30.43 & 61.09 & \textcolor{darkgreen}{30.66} \\
Entropy & 53.45 & 5.97 & \textcolor{darkgreen}{-47.47} & 73.51 & 73.87 & \textcolor{darkgreen}{0.37} & 42.20 & 2.05 & \textcolor{darkgreen}{-40.15} & 69.04 & 67.19 & \textcolor{darkred}{-1.85} & 72.57 & 3.70 & \textcolor{darkgreen}{-68.87} & 33.93 & 63.13 & \textcolor{darkgreen}{29.20} \\
Energy & 16.97 & 4.05 & \textcolor{darkgreen}{-12.91} & 45.89 & 49.65 & \textcolor{darkgreen}{3.75} & 11.61 & 3.31 & \textcolor{darkgreen}{-8.30} & 50.25 & 54.63 & \textcolor{darkgreen}{4.39} & 41.70 & 4.20 & \textcolor{darkgreen}{-37.51} & 55.28 & 59.13 & \textcolor{darkgreen}{3.85} \\
P(True) & 41.20 & 5.92 & \textcolor{darkgreen}{-35.28} & 63.13 & 63.33 & \textcolor{darkgreen}{0.20} & 41.85 & 1.83 & \textcolor{darkgreen}{-40.03} & 51.40 & 55.46 & \textcolor{darkgreen}{4.06} & 56.48 & 2.91 & \textcolor{darkgreen}{-53.57} & 83.91 & 83.33 & \textcolor{darkred}{-0.57} \\
VC & 35.48 & 14.58 & \textcolor{darkgreen}{-20.90} & 83.27 & 82.79 & \textcolor{darkred}{-0.48} & 57.45 & 1.34 & \textcolor{darkgreen}{-56.11} & 63.15 & 63.71 & \textcolor{darkgreen}{0.56} & 20.64 & 27.54 & \textcolor{darkred}{6.90} & 82.18 & 81.87 & \textcolor{darkred}{-0.31} \\
CoE-C & 8.24 & 8.60 & \textcolor{darkred}{0.36} & 70.14 & 70.39 & \textcolor{darkgreen}{0.25} & 3.35 & 7.64 & \textcolor{darkred}{4.29} & 53.71 & 54.83 & \textcolor{darkgreen}{1.12} & 20.65 & 3.95 & \textcolor{darkgreen}{-16.70} & 74.62 & 76.55 & \textcolor{darkgreen}{1.92} \\
CoE-R & 19.76 & 6.69 & \textcolor{darkgreen}{-13.07} & 52.95 & 65.96 & \textcolor{darkgreen}{13.01} & 29.21 & 5.97 & \textcolor{darkgreen}{-23.25} & 37.19 & 62.35 & \textcolor{darkgreen}{25.16} & 4.89 & 3.40 & \textcolor{darkgreen}{-1.48} & 73.81 & 81.41 & \textcolor{darkgreen}{7.60} \\
Eigenscore & 43.99 & 2.40 & \textcolor{darkgreen}{-41.59} & 74.18 & 74.17 & \textcolor{darkred}{-0.01} & 42.04 & 6.77 & \textcolor{darkgreen}{-35.27} & 68.53 & 68.58 & \textcolor{darkgreen}{0.05} & 66.49 & 2.86 & \textcolor{darkgreen}{-63.63} & 35.30 & 57.91 & \textcolor{darkgreen}{22.61} \\
SE & 66.82 & 5.27 & \textcolor{darkgreen}{-61.55} & 49.77 & 65.13 & \textcolor{darkgreen}{15.37} & 50.65 & 5.00 & \textcolor{darkgreen}{-45.65} & 44.31 & 69.38 & \textcolor{darkgreen}{25.07} & 80.22 & 4.20 & \textcolor{darkgreen}{-76.03} & 36.77 & 42.18 & \textcolor{darkgreen}{5.40} \\

\midrule
\headercolor
\multicolumn{19}{c}{\textbf{Ministral-3-8B}} \\

MSP & 28.86 & 12.85 & \textcolor{darkgreen}{-16.02} & 83.32 & 83.32 & 0.00 & 42.17 & 10.10 & \textcolor{darkgreen}{-32.07} & 73.76 & 73.76 & 0.00 & 47.57 & 31.05 & \textcolor{darkgreen}{-16.52} & 87.37 & 87.37 & 0.00 \\
Perplexity & 39.57 & 15.84 & \textcolor{darkgreen}{-23.73} & 81.88 & 81.88 & 0.00 & 46.23 & 8.44 & \textcolor{darkgreen}{-37.79} & 74.05 & 74.05 & 0.00 & 58.70 & 23.57 & \textcolor{darkgreen}{-35.13} & 83.32 & 83.32 & 0.00 \\
Entropy & 25.31 & 14.32 & \textcolor{darkgreen}{-10.99} & 83.31 & 83.31 & 0.00 & 28.09 & 12.75 & \textcolor{darkgreen}{-15.34} & 77.08 & 77.08 & 0.00 & 34.80 & 26.41 & \textcolor{darkgreen}{-8.39} & 87.10 & 87.10 & 0.00 \\
Energy & 7.47 & 13.21 & \textcolor{darkred}{5.74} & 80.45 & 80.45 & 0.00 & 17.22 & 6.20 & \textcolor{darkgreen}{-11.02} & 76.23 & 75.59 & \textcolor{darkred}{-0.64} & 18.78 & 28.35 & \textcolor{darkred}{9.57} & 87.61 & 87.61 & 0.00 \\
P(True) & 5.33 & 8.59 & \textcolor{darkred}{3.27} & 63.21 & 63.58 & \textcolor{darkgreen}{0.37} & 17.06 & 2.90 & \textcolor{darkgreen}{-14.16} & 49.73 & 51.38 & \textcolor{darkgreen}{1.65} & 31.52 & 4.15 & \textcolor{darkgreen}{-27.37} & 69.46 & 61.27 & \textcolor{darkred}{-8.19} \\
VC & 39.54 & 2.17 & \textcolor{darkgreen}{-37.37} & 64.84 & 64.84 & 0.00 & 58.13 & 0.32 & \textcolor{darkgreen}{-57.81} & 69.51 & 69.51 & 0.00 & 29.83 & 2.00 & \textcolor{darkgreen}{-27.83} & 88.35 & 88.17 & \textcolor{darkred}{-0.18} \\
CoE-C & 43.86 & 0.78 & \textcolor{darkgreen}{-43.08} & 60.26 & 58.85 & \textcolor{darkred}{-1.41} & 33.94 & 3.32 & \textcolor{darkgreen}{-30.62} & 62.13 & 63.06 & \textcolor{darkgreen}{0.93} & 12.84 & 6.49 & \textcolor{darkgreen}{-6.34} & 52.53 & 63.85 & \textcolor{darkgreen}{11.32} \\
CoE-R & 19.17 & 6.00 & \textcolor{darkgreen}{-13.17} & 48.34 & 61.70 & \textcolor{darkgreen}{13.35} & 23.68 & 5.85 & \textcolor{darkgreen}{-17.84} & 34.92 & 65.08 & \textcolor{darkgreen}{30.16} & 25.98 & 4.86 & \textcolor{darkgreen}{-21.12} & 72.44 & 80.63 & \textcolor{darkgreen}{8.19} \\
Eigenscore & 14.11 & 12.36 & \textcolor{darkgreen}{-1.75} & 80.90 & 80.90 & 0.00 & 7.39 & 4.27 & \textcolor{darkgreen}{-3.12} & 81.08 & 82.29 & \textcolor{darkgreen}{1.21} & 18.35 & 5.00 & \textcolor{darkgreen}{-13.35} & 81.92 & 84.70 & \textcolor{darkgreen}{2.79} \\
SE & 47.58 & 9.89 & \textcolor{darkgreen}{-37.69} & 55.39 & 74.68 & \textcolor{darkgreen}{19.30} & 57.45 & 10.34 & \textcolor{darkgreen}{-47.11} & 45.57 & 67.02 & \textcolor{darkgreen}{21.45} & 74.90 & 5.93 & \textcolor{darkgreen}{-68.97} & 61.48 & 74.15 & \textcolor{darkgreen}{12.67} \\

\midrule
\headercolor
\multicolumn{19}{c}{\textbf{Gemma-2-9B}} \\

MSP & 31.96 & 10.67 & \textcolor{darkgreen}{-21.29} & 71.46 & 71.46 & 0.00 & 38.69 & 7.99 & \textcolor{darkgreen}{-30.70} & 72.95 & 72.91 & \textcolor{darkred}{-0.04} & 65.10 & 2.64 & \textcolor{darkgreen}{-62.46} & 67.70 & 67.31 & \textcolor{darkred}{-0.39} \\
Perplexity & 29.55 & 6.22 & \textcolor{darkgreen}{-23.33} & 70.71 & 70.14 & \textcolor{darkred}{-0.57} & 29.03 & 3.37 & \textcolor{darkgreen}{-25.65} & 70.40 & 69.45 & \textcolor{darkred}{-0.95} & 60.10 & 6.65 & \textcolor{darkgreen}{-53.44} & 67.97 & 67.97 & 0.00 \\
Entropy & 22.97 & 7.52 & \textcolor{darkgreen}{-15.45} & 72.84 & 72.91 & \textcolor{darkgreen}{0.07} & 21.02 & 9.47 & \textcolor{darkgreen}{-11.55} & 73.99 & 73.77 & \textcolor{darkred}{-0.22} & 46.75 & 17.07 & \textcolor{darkgreen}{-29.67} & 79.14 & 79.14 & 0.00 \\
Energy & 13.43 & 4.38 & \textcolor{darkgreen}{-9.04} & 67.11 & 67.28 & \textcolor{darkgreen}{0.17} & 24.46 & 7.23 & \textcolor{darkgreen}{-17.23} & 54.84 & 51.42 & \textcolor{darkred}{-3.42} & 64.72 & 2.94 & \textcolor{darkgreen}{-61.77} & 39.41 & 63.13 & \textcolor{darkgreen}{23.71} \\
P(True) & 31.91 & 2.16 & \textcolor{darkgreen}{-29.75} & 60.19 & 60.05 & \textcolor{darkred}{-0.13} & 38.10 & 2.61 & \textcolor{darkgreen}{-35.49} & 48.23 & 48.08 & \textcolor{darkred}{-0.15} & 66.03 & 18.19 & \textcolor{darkgreen}{-47.84} & 81.88 & 81.88 & 0.00 \\
VC & 34.10 & 3.31 & \textcolor{darkgreen}{-30.79} & 58.60 & 58.60 & 0.00 & 44.26 & 2.02 & \textcolor{darkgreen}{-42.24} & 63.56 & 63.33 & \textcolor{darkred}{-0.23} & 63.41 & 3.98 & \textcolor{darkgreen}{-59.43} & 47.78 & 68.99 & \textcolor{darkgreen}{21.20} \\
CoE-C & 25.20 & 6.39 & \textcolor{darkgreen}{-18.82} & 48.72 & 63.24 & \textcolor{darkgreen}{14.52} & 18.16 & 1.22 & \textcolor{darkgreen}{-16.94} & 49.03 & 53.22 & \textcolor{darkgreen}{4.19} & 49.59 & 1.91 & \textcolor{darkgreen}{-47.68} & 17.52 & 82.21 & \textcolor{darkgreen}{64.69} \\
CoE-R & 9.77 & 4.49 & \textcolor{darkgreen}{-5.28} & 49.71 & 60.61 & \textcolor{darkgreen}{10.90} & 7.28 & 2.43 & \textcolor{darkgreen}{-4.85} & 49.18 & 53.22 & \textcolor{darkgreen}{4.04} & 44.37 & 5.95 & \textcolor{darkgreen}{-38.43} & 18.54 & 81.46 & \textcolor{darkgreen}{62.92} \\
Eigenscore & 12.88 & 5.17 & \textcolor{darkgreen}{-7.71} & 70.64 & 70.37 & \textcolor{darkred}{-0.27} & 14.32 & 5.85 & \textcolor{darkgreen}{-8.47} & 70.41 & 70.60 & \textcolor{darkgreen}{0.19} & 38.94 & 1.91 & \textcolor{darkgreen}{-37.03} & 59.24 & 61.71 & \textcolor{darkgreen}{2.47} \\
Semantic Entropy & 35.34 & 5.71 & \textcolor{darkgreen}{-29.63} & 41.17 & 61.24 & \textcolor{darkgreen}{20.07} & 39.09 & 3.54 & \textcolor{darkgreen}{-35.56} & 48.73 & 60.41 & \textcolor{darkgreen}{11.68} & 69.70 & 4.14 & \textcolor{darkgreen}{-65.57} & 53.41 & 60.96 & \textcolor{darkgreen}{7.56} \\

\bottomrule
\end{tabular}
}

\caption{Main results. $\Delta$ denotes the change of ECE and AUC after truth anchoring (TAC), with improvements in \textcolor{darkgreen}{green}.}

\label{tab:main_result}
\end{table*}

\subsection{Proxy Score Truth Anchoring}

A particularly simple instantiation arises when the available information is just the scalar uncertainty score $S$, such as predictive entropy or response consistency. In this case, the Bayes-optimal target reduces to $p^\star(S)=P(C=1|S)$, namely, the posterior correctness conditioned on the score itself.
This suggests a minimal post-hoc calibration strategy.
Rather than using $S$ directly as uncertainty with respect to the truth, we learn a mapping

\begin{equation}
    m:\mathbb{R} \to [0,1]
\end{equation}
from the raw score to its truth-aligned estimate.

\subsection{Implementation}

\paragraph{Mapper architecture.} We instantiate $m_\theta$ as a lightweight multilayer perceptron (MLP) operating on the one-dimensional input $S$.
Concretely, the mapper consists of an input layer, three hidden layers with ReLU activations, and a scalar output head followed by a sigmoid function.
From a raw score $S_i$, the mapper produces $\hat{p}_i= m_\theta(S_i) \in (0,1)$.
In practice, a small network is sufficiently capable as the input is one-dimensional. In our experiments, we use a hidden dimension of $32$ with three hidden layers, and optimise with Adam using a learning rate of $0.01$, with early stopping on validation AUC.

\paragraph{Objective.}
Given score--label pairs $\{S_i, C_i\}_{i=1}^n$, we optimise the binary cross-entropy (BCE) loss as the main objective,
\begin{equation}
\begin{split}
    \mathcal{L}_{\mathrm{BCE}} = & \frac{1}{n} \sum_{i=1}^n \ell\big(m_\theta(S_i),\, C_i\big) \\
    = & -\frac{1}{n}\sum_{i=1}^n\Big[ C_i \log m_\theta(S_i) \\
    & + \; (1-C_i) \log \big(1-m_\theta(S_i)\big)\Big]
\end{split}
\end{equation}
$\mathcal{L}_{\mathrm{BCE}}$ is a strictly proper scoring loss, so its population optimum is the correctness posterior $p^\star(S)$.

Since a primary evaluation metric is AUC, we also consider augmenting the objective with an auxiliary pairwise ranking loss, which is widely used to improve ordering quality and is closely connected to AUC-oriented learning \citep{zhe2007learning}.
Specifically, for positive--negative pairs $(i,j)$ with $C_i=1$ and $C_j=0$, we use a logistic pairwise loss

\begin{multline} 
\mathcal{L}_{\mathrm{rank}}(\theta) = \frac{1}{|\mathcal{P}|} \sum_{(i,j)\in\mathcal{P}}\log\!\Big(1\; + \\
\exp\big(-(z_i-z_j)\big)\Big),
\end{multline}
where $z_i$ and $z_j$ are pre-sigmoid logits of the mapper and $\mathcal{P}=\{(i,j): C_i=1,\, C_j=0\}$ is the set of positive--negative pairs.
Intuitively, this term encourages correct responses to receive higher scores than incorrect ones.

Our final objective is

\begin{equation}
    \mathcal{L}(\theta) = \mathcal{L}_{\mathrm{BCE}}(\theta) + \phi_{\mathrm{rank}} \, \mathcal{L}_{\mathrm{rank}}(\theta),
\end{equation}
where $\phi_{\mathrm{rank}}$ controls the strength of the ranking regulariser. By default, we use $\phi_{\mathrm{rank}}=1$.
We perform sensitivity analysis of $\phi_\mathrm{rank}$ in \Cref{subsec:sensitivity}.

We note that the ranking term is auxiliary. Unlike BCE, it is not a strictly proper scoring rule, and does not by itself target the correctness posterior.
We include it to better align training with AUC. Collectively, the BCE term provides probabilistic grounding, while the pairwise term can help sharpen ranking when raw proxy is informative but poorly ordered.
\paragraph{Remark.}  We emphasise that this is intentionally minimal, with the mapper operating only on the raw uncertainty, without introducing additional model features. As such, it isolates the effect of truth anchoring itself. If performance improves after calibration, this directly indicates that the original proxy was not intrinsically truth-aligned but could be made more useful once connected to correctness.
Furthermore, TAC mainly helps with calibration and does not overturn the non-discriminability result of \Cref{sec:study}. If a score carries no correctness information, strong discrimination cannot be recovered by post-hoc mapping of that score alone but may be achieved with more information, \textit{e.g.} through inter-score anchoring in \Cref{fig:inter}.

\begin{figure}[!ht]
    \centering
    \begin{subfigure}{\linewidth}
        \includegraphics[width=\linewidth]{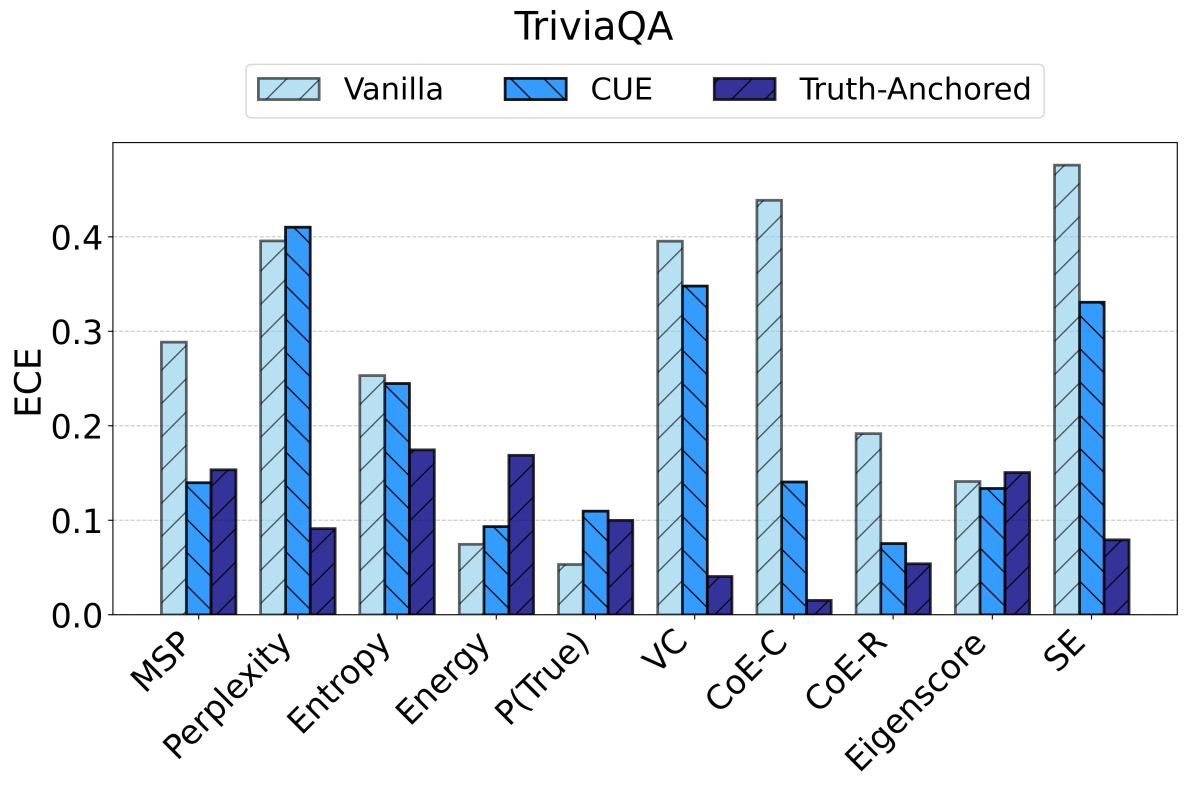}
    \end{subfigure}
    
    \begin{subfigure}{\linewidth}
        \includegraphics[width=\linewidth]{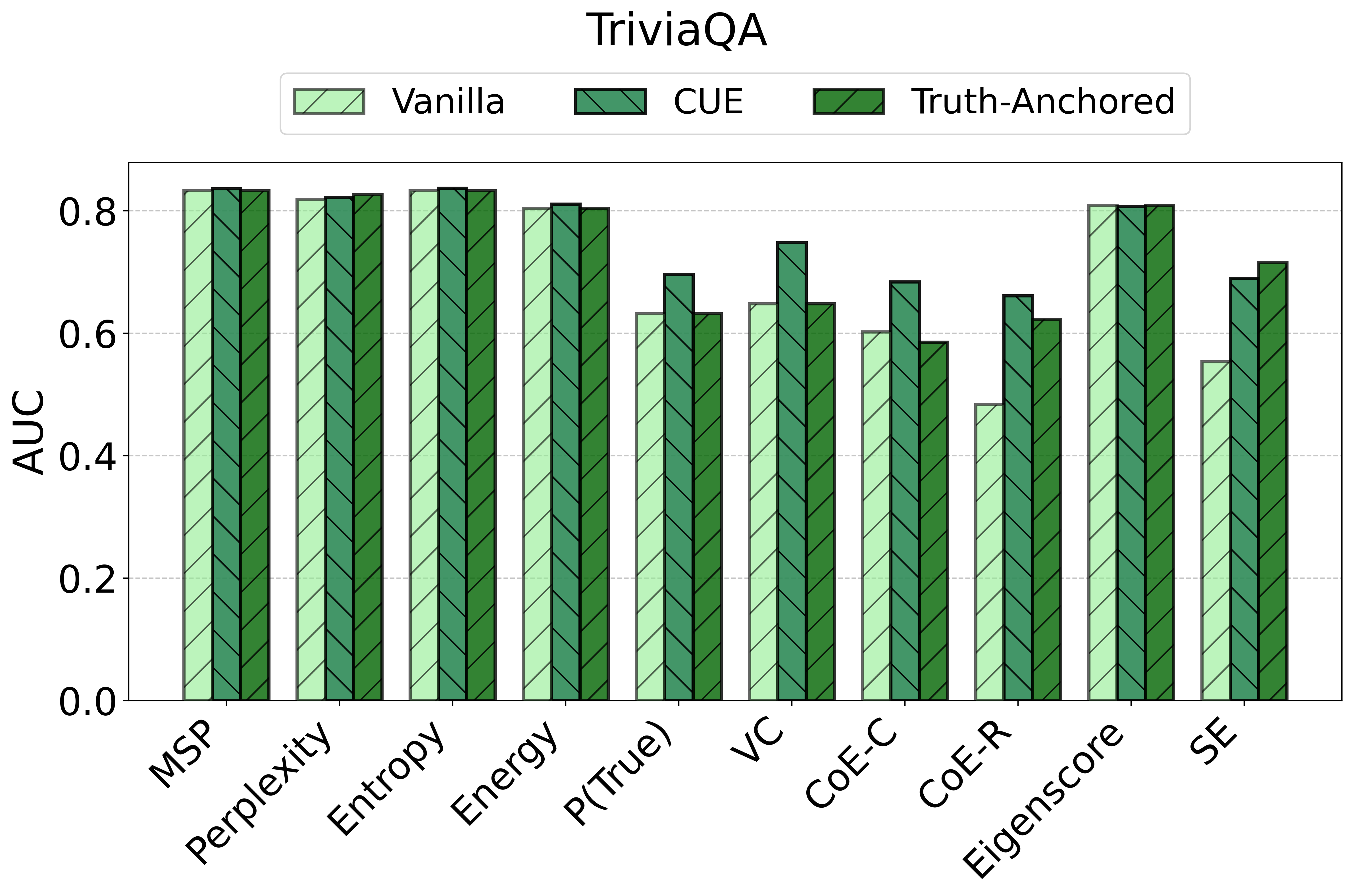}
    \end{subfigure}
    \caption{Performance of vanilla, CUE, and truth-anchored scores.}
    \label{fig:cue_trivia}
\end{figure}

\section{Experimental Setup}
\label{sec:expt}

\subsection{Language Models}
We select five models from four popular model families: \textbf{Qwen-3-4B} \citep{yang2025qwen3}, \textbf{Llama-3.2-3B}, \textbf{Llama-3.1-8B} \citep{grattafiori2024llama}, \textbf{Ministral-3-8B} \citep{liu2026ministral}, and \textbf{Gemma-2-9B} \citep{team2024gemma} for our main experiments. We report the performance of the Llama models in Appendix~\ref{app:add_expt} for better readability.
For more detailed experiments and analysis, we primarily utilise Ministral-3-8B.
We supplement the experiments with more specialised tasks in \Cref{subsec:spec_data} with Gemma-2-9B evaluation.
For all models, we work with their instruction-tuned versions.

\subsection{Datasets.}
We utilise three representative open-domain question-answering (QA) datasets, covering a wide spectrum of topics:
\begin{inparaenum}[(i)]
    \item \textbf{TriviaQA} \citep{joshi-etal-2017-triviaqa}: broad, general-knowledge factual recall;
    \item \textbf{SciQ} \citep{welbl-etal-2017-crowdsourcing}: domain-specific scientific queries; and 
    \item \textbf{PopQA} \citep{mallen-etal-2023-trust}: long-tail entity questions.
\end{inparaenum}
Collectively, these datasets constitute fact-heavy queries across many topics, making them suitable to test factual recall of language models.

\subsection{Metrics}
We evaluate our results using two standard metrics:
\begin{inparaenum}[(i)]
    \item \textbf{AUC} \citep{davis2006relationship}: with a value of 1 signifying perfect discriminative performance to rank samples reliably, while a value of 0.5 indicates that estimation is no better than chance; and
    \item \textbf{ECE} \citep[Expected Calibration Error;][]{guo2017calibration}, which assesses the calibration by calculating the expected difference between predicted confidence and empirical accuracy, penalizing over- and under-confidence.
    In practice, this is computed by partitioning all $n$ confidence scores into $M$ bins and comparing the average bin confidence ($\mathrm{conf}$) of each bin $B_m$ against the ratio of correct predictions ($\mathrm{acc}$).
    We follow the standard formulation of ECE \citep{guo2017calibration} with 10 equal-width bins.
\end{inparaenum}
\begin{equation}
    \mathrm{ECE} = \sum_{m=1}^{M} \frac{|B_M|}{n} |\mathrm{acc}(B_m) - \mathrm{conf}(B_m)|
\end{equation}
An $\mathrm{ECE}$ of 0 means perfect calibration, where the confidence scores exactly match actual empirical accuracy.
We note a degenerate case where ECE is zero but AUC is 0.5 (random chance). This happens when a model makes uninformative but perfectly calibrated predictions.
Taken together, these two metrics provide a more complete and fairer evaluation of confidence or uncertainty scores.

\subsection{Baselines}
We evaluate 10 UE scores across four families of approaches:
\begin{inparaenum}[(i)]
    \item \textbf{Logit-based}: Maximum Sequence Probability (MSP),  Perplexity \citep{si2022prompting}, Entropy \citep{huang2023look}, Energy \citep{liu2020energy};
    \item \textbf{Elicitation}: P(True) \citep{kadavath2022language}, Verbalised Confidence \citep[VC;][]{tian-etal-2023-just};
    \item \textbf{Internal States}: CoE-C, CoE-R \citep{wang2024latent}; and
    \item \textbf{Consistency}: Eigenscore \citep[ES;][]{chen2024inside} and Semantic Entropy \citep[SE;][]{kuhn2023semantic}.
\end{inparaenum}
Furthermore, we compare with the correction framework CUE \citep{li-etal-2025-towards}.
For CUE, we make use of DeBERTa V3~\citep{he2021deberta, he2021debertav3} as the encoder for the corrector model and train for three epochs.

\section{Results and Analysis}

\subsection{Main Results}
\Cref{tab:main_result} and \Cref{fig:ece_comparison} present the main results.
First, we observe the proxy failure phenomenon. For example, with Ministral-3-8B, CoE-R and SE fail to discriminate meaningfully on SciQ, although they perform significantly better on TriviaQA and PopQA, highlighting how proxy scores \emph{can}, but often \emph{fail to}, track factuality.
We observe a similar occurrence with Gemma-2-9B, with CoE-C and CoE-R performing worse than chance on PopQA, although their AUCs with other models are competitive.
These observations highlight the prevalence of proxy failure, where performance varies considerably across models and datasets.

Second, most vanilla scores exhibit poor ECE, undermining their utility as uncertainty indicators.
Most importantly, TAC consistently leads to significant reductions in ECE, and frequently improves AUC.
For instance, on average, with Ministral-3-8B, TAC leads to a change of $-17.48, -26.69, -21.55$ in ECE, $+3.16, +5.48, +2.66$ in AUC, on TriviaQA, SciQ, and PopQA, respectively.
We verify the statistical significance of these gains in \Cref{tab:stat_sig}.
\Cref{tab:examples} illustrates qualitative examples of proxy failure with entropy-derived confidence scores on LLM-generated answers and how we apply TAC as a calibration layer.
These results validate TAC as an important post-hoc calibration step.

\subsection{Comparison with CUE}

\begin{table}
\centering
\renewcommand{\arraystretch}{1.1}
\setlength{\tabcolsep}{2pt}
\begin{subtable}{\linewidth}
    \resizebox{\linewidth}{!}{
    \begin{tabular}{l ccccc|c}
    \toprule
    
    \headercolorgray
    \textbf{\# Labels} & \textbf{8} & \textbf{16} & \textbf{32} & \textbf{64} & \textbf{128} & \textbf{All} \\
    \textbf{$\Delta$(ECE) ($\downarrow$}) & -12.31 & -12.46 & -12.75 & -12.74 & -16.09 & -17.48 \\
    \textbf{$\Delta$(AUC) ($\uparrow$}) & -7.39 & -2.67 & -0.76 & -0.37 & -0.96 & 3.16 \\
    
    \midrule
    \headercolorgray
    \textbf{\% Corrupt} & \textbf{0.5} & \textbf{0.4} & \textbf{0.3} & \textbf{0.2} & \textbf{0.1} & \textbf{Clean} \\
    \textbf{$\Delta$(ECE) ($\downarrow$}) & -22.81 & -18.95 & -17.96 & -18.54 & -18.78 & -17.48 \\
    \textbf{$\Delta$(AUC) ($\uparrow$}) & -31.15 & -3.32 & -0.62 & 1.49 & 2.76 & 3.16 \\
    
    \bottomrule
    \end{tabular}
    }
    \caption{TriviaQA}
\end{subtable}

\vskip\floatsep

\begin{subtable}{\linewidth}
    \resizebox{\linewidth}{!}{
    \begin{tabular}{l ccccc|c}
    \toprule
    
    \headercolorgray
    \textbf{\# Labels} & \textbf{8} & \textbf{16} & \textbf{32} & \textbf{64} & \textbf{128} & \textbf{All} \\
    \textbf{$\Delta$(ECE) ($\downarrow$}) & -18.74 & -19.30 & -22.08 & -22.86 & -18.59 & -26.69 \\
    \textbf{$\Delta$(AUC) ($\uparrow$}) & 2.03 & -0.38 & -0.79 & 4.66 & 4.86 & 5.48 \\
    \midrule
    \headercolorgray
    \textbf{\% Corrupt} & \textbf{0.5} & \textbf{0.4} & \textbf{0.3} & \textbf{0.2} & \textbf{0.1} & \textbf{Clean} \\
    \textbf{$\Delta$(ECE) ($\downarrow$}) & -22.27 & -21.08 & -23.62 & -25.45 & -25.72 &  -26.69 \\
    \textbf{$\Delta$(AUC) ($\uparrow$}) & -14.71 & -0.56 & 3.89 & 4.15 & 4.64 & 5.48 \\
    
    \bottomrule
    \end{tabular}
    }
    \caption{SciQ}
\end{subtable}

\vskip\floatsep

\begin{subtable}{\linewidth}
    \resizebox{\linewidth}{!}{
    \begin{tabular}{l ccccc|c}
    \toprule
    
    \headercolorgray
    \textbf{\# Labels} & \textbf{8} & \textbf{16} & \textbf{32} & \textbf{64} & \textbf{128} & \textbf{All} \\
    \textbf{$\Delta$(ECE) ($\downarrow$}) & -16.83 & -20.02 & -11.58 & -6.18 & -12.97 & -21.55 \\
    \textbf{$\Delta$(AUC) ($\uparrow$}) & -51.30 & -0.09 & -6.13 & -2.58 & 0.25 & 2.66 \\
    
    \midrule
    \headercolorgray
    \textbf{\% Corrupt} & \textbf{0.5} & \textbf{0.4} & \textbf{0.3} & \textbf{0.2} & \textbf{0.1} & \textbf{Clean} \\
    \textbf{$\Delta$(ECE) ($\downarrow$}) & -5.60 & -7.06 & -12.63 & -15.64 & -24.23 & -21.55 \\
    \textbf{$\Delta$(AUC) ($\uparrow$}) & -46.52 & -4.62 & 1.30 & 2.75 & 2.83 & 2.66 \\
    
    \bottomrule
    \end{tabular}
    }
    \caption{PopQA}
\end{subtable}

\caption{Performance (average $\Delta$) under few-shot and noisy label constraints.}

\label{tab:nosiy}
\end{table}

In \Cref{fig:cue_trivia}, we compare TAC against CUE.
\Cref{fig:cue_sciq,fig:cue_popqa} show results for SciQ and PopQA, respectively.
On average, with Ministral-3-8B, TAC reports a reduction in ECE of $-21.91$ and an increase in AUC of $+3.77$, compared to CUE's $-14.03$ and $+7.26$.
Overall, TAC produces considerably better calibrated scores than CUE, without access to BERT-based embeddings, while CUE has a greater advantage in recovering AUC.
This trade-off arises because TAC operates only on a scalar proxy and is designed mainly for calibration, whereas CUE utilises the additional information in BERT embeddings that can improve discriminability.
On balance, this highlights TAC's strengths in aligning raw scores to achieve better calibrated, and thus, more meaningful uncertainty estimates.

\subsection{Noisy and Few-shot Supervision}
\Cref{tab:nosiy} summarises the average change in ECE and AUC after TAC with few-shot or noisy labels.
Individual evaluation metrics are deferred to \Cref{fig:ds,fig:cr}.
Calibration improvements persist with limited and noisy supervision, although AUC gains can weaken in the few-shot regime.
This hints at TAC as a realistic calibration protocol that does not demand perfect supervision.

\begin{table}[htbp]
\centering
\setlength{\tabcolsep}{1mm}
\renewcommand{\arraystretch}{1}

\begin{subtable}{\linewidth}
    \resizebox{\linewidth}{!}{
    \centering
    \begin{tabular}{@{} l| *{3}{c} @{}}
         \toprule
         \multirow{2}{*}{\textbf{Train on}} & \multicolumn{3}{c}{\textbf{Test on}} \\
         \cmidrule(l){2-4}
         & TriviaQA & SciQ & PopQA \\
         \midrule 
         TriviaQA & \makecell[c]{9.60 / 73.35\\\small\textcolor{darkgreen}{(-17.48)} / \small\textcolor{darkgreen}{(+3.16)}} & \makecell[c]{13.30 / 69.40\\\small\textcolor{darkgreen}{(-19.84)} / \small\textcolor{darkgreen}{(+5.00)}} & \makecell{25.98 / 77.35 \\\small\textcolor{darkgreen}{(-9.35)} / \small\textcolor{darkgreen}{(+0.19)}}\\
         \midrule
         SciQ & \makecell[c]{12.58 / 71.40 \\\small\textcolor{darkgreen}{(-14.50)} / \small\textcolor{darkgreen}{(+1.21)}} & \makecell[c]{6.45 / 69.88 \\\small\textcolor{darkgreen}{(-26.69)} / \small\textcolor{darkgreen}{(+5.48)}} & \makecell[c]{16.73 / 71.12 \\\small\textcolor{darkgreen}{(-18.60)} / \small\textcolor{darkred}{(-6.04)}}\\
         \midrule
         PopQA & \makecell[c]{20.56 / 70.35\\\small\textcolor{darkgreen}{(-6.52) / \small\textcolor{darkgreen}{(+0.16)}}} & \makecell[c]{18.76 / 64.12 \\\small\textcolor{darkgreen}{(-14.44)} / \small\textcolor{darkred}{(-0.28)}} &  \makecell[c]{13.78 / 79.82 \\\small\textcolor{darkgreen}{(-21.55)} / \small\textcolor{darkgreen}{(+2.66)}} \\
         \bottomrule
    \end{tabular}
    }
    \caption{TAC}
\end{subtable}

\vskip\floatsep

\begin{subtable}{\linewidth}
    \resizebox{\linewidth}{!}{
    \centering
    \begin{tabular}{@{} l| *{3}{c} @{}}
         \toprule
         \multirow{2}{*}{\textbf{Train on}} & \multicolumn{3}{c}{\textbf{Test on}} \\
         \cmidrule(l){2-4}
         & TriviaQA & SciQ & PopQA \\
         \midrule 
         TriviaQA & \makecell[c]{15.79 / 76.10\\\small\textcolor{darkgreen}{(-11.29)} / \small\textcolor{darkgreen}{(+5.91)}} & \makecell[c]{31.57 / 65.66\\\small\textcolor{darkgreen}{(-1.56)} / \small\textcolor{darkgreen}{(+1.26)}} & \makecell{ 31.68 / 83.32 \\\small\textcolor{darkgreen}{(-3.65)} / \small\textcolor{darkgreen}{(+6.16)}}\\
         \midrule
         SciQ & \makecell[c]{ 4.65 / 66.45 \\\small\textcolor{darkgreen}{(-22.43)} / \small\textcolor{darkred}{(-3.74)}} & \makecell[c]{ 20.03 / 73.37\\\small\textcolor{darkgreen}{(-13.11)} / \small\textcolor{darkgreen}{(+8.97)}} & \makecell{23.39 / 73.06\\\small\textcolor{darkgreen}{(-11.93)} / \small\textcolor{darkred}{(-4.10)}} \\
         \midrule
         PopQA & \makecell[c]{ 18.61 / 70.62 \\\small\textcolor{darkgreen}{(-8.47) / \small\textcolor{darkgreen}{(+0.43)}}} & \makecell[c]{ 19.51 / 62.96 \\\small\textcolor{darkgreen}{(-13.63)} / \small\textcolor{darkred}{(-1.44)}} &  \makecell[c]{ 17.65 / 84.06 \\\small\textcolor{darkgreen}{(-17.68)} / \small\textcolor{darkgreen}{(+6.90)}} \\
         \bottomrule
    \end{tabular}
    }
    \caption{CUE}
\end{subtable}

\caption{Average performance and $\Delta$ compared to vanilla scores under transfer settings.}
\label{tab:transfer}
\end{table}

\begin{figure}
    \centering
    \includegraphics[width=\linewidth]{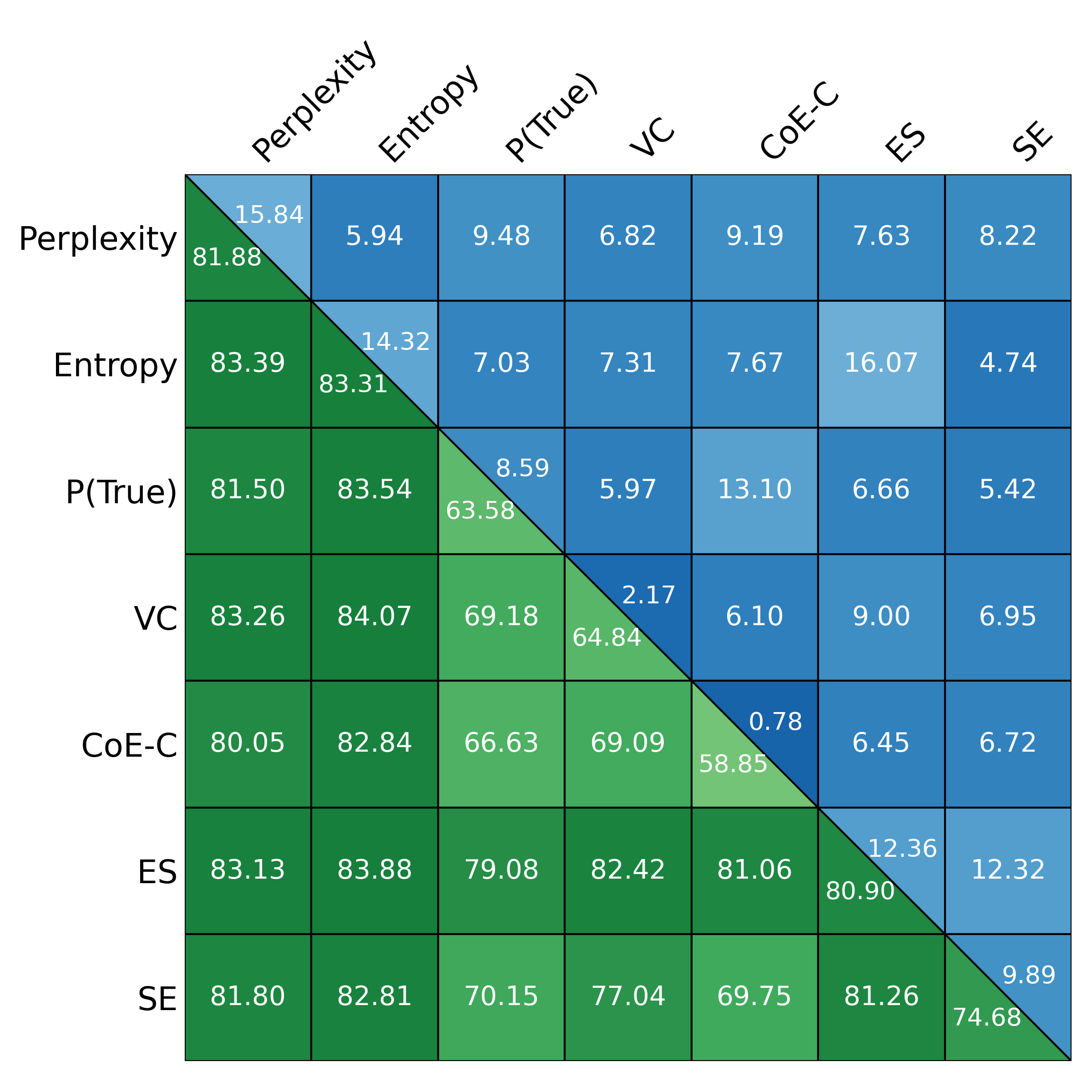}
    \caption{Pairwise TAC. Upper triangle shows ECE, and lower triangle shows AUC. Darker colours indicate more favourable results.}
    \label{fig:inter}
\end{figure}

\subsection{Cross-Task Transfer}
Complementing the noisy and few-shot constraints, we further consider another realistic out-of-distribution (OOD) setting, where the scores and correctness labels are obtained from a different query set.
Under OOD constraints, TAC reports average changes in ECE / AUC of $-13.88 / +0.24$, compared to CUE's $-10.28 / -0.23$, over the vanilla scores.
We break down the average changes by each individual train--test configuration in \Cref{tab:transfer}.
TAC remains useful under train-test mismatch, although gains are weaker than in-domain calibration. The transfer results are not uniform: transfer between TriviaQA and SciQ is relatively stable, especially when trained on TriviaQA and tested on SciQ.
In contrast, transfer involving PopQA is weaker. For example, training on SciQ and testing on PopQA reduces AUC by $6.04$ points. Similarly, using PopQA as the source gives weaker transfer to TriviaQA and SciQ than in-domain training. This is likely due to PopQA’s long-tail entity distribution, which represents a larger shift from TriviaQA and SciQ, than the shift between the latter two datasets.

Overall, this suggests that the learned calibration is partially transferable across related QA tasks, especially for calibration, but weakens as the shift between train and test distributions widens. Therefore, TAC can be regarded as a task-sensitive calibration. When possible, in-domain calibration is preferable, while cross-task TAC remains a practical option when target labels are unavailable.

\begin{figure}[htbp]
    \centering
    \begin{subfigure}{\linewidth}
        \includegraphics[width=\linewidth]{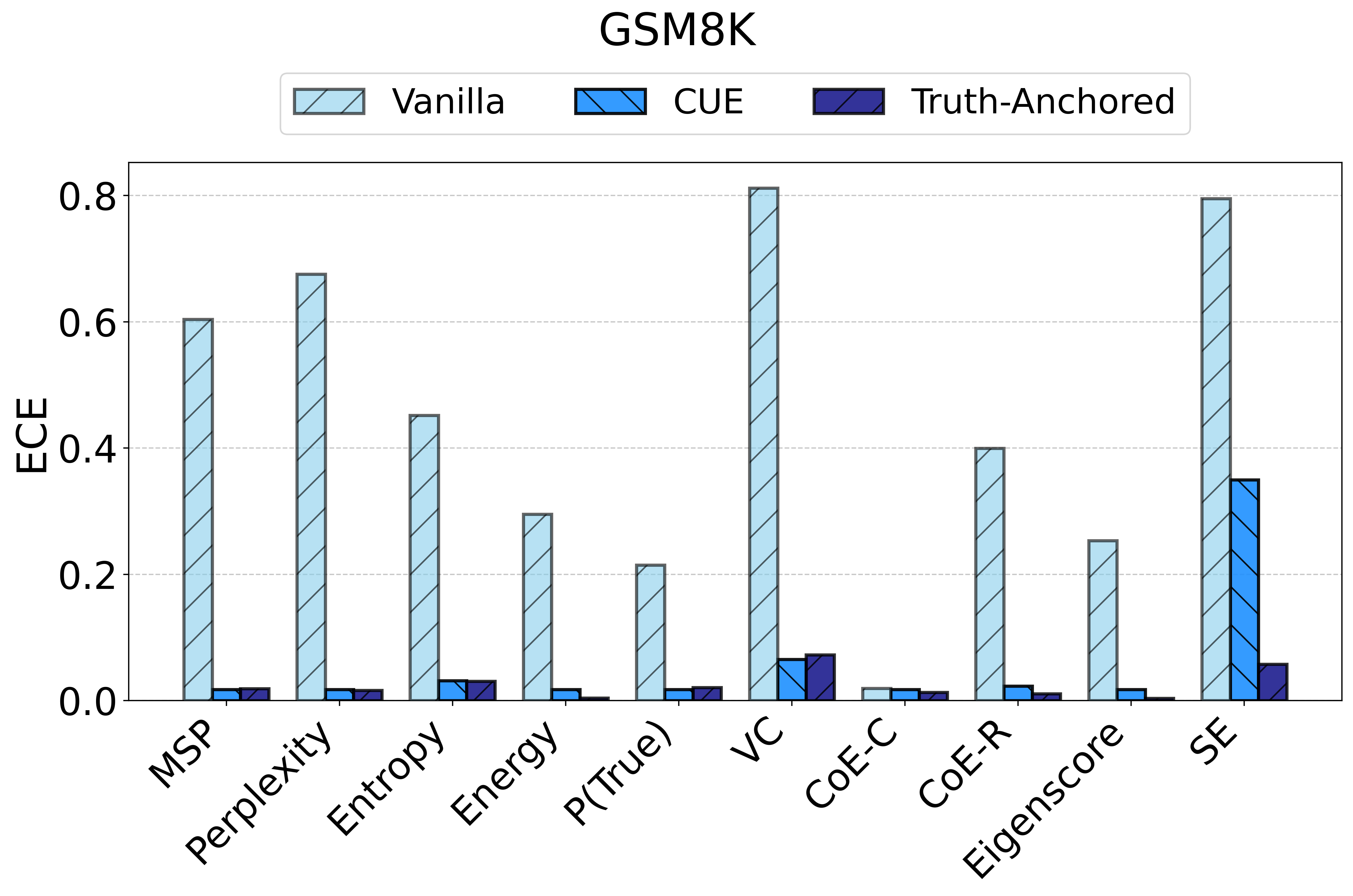}
    \end{subfigure}
    
    \begin{subfigure}{\linewidth}
        \includegraphics[width=\linewidth]{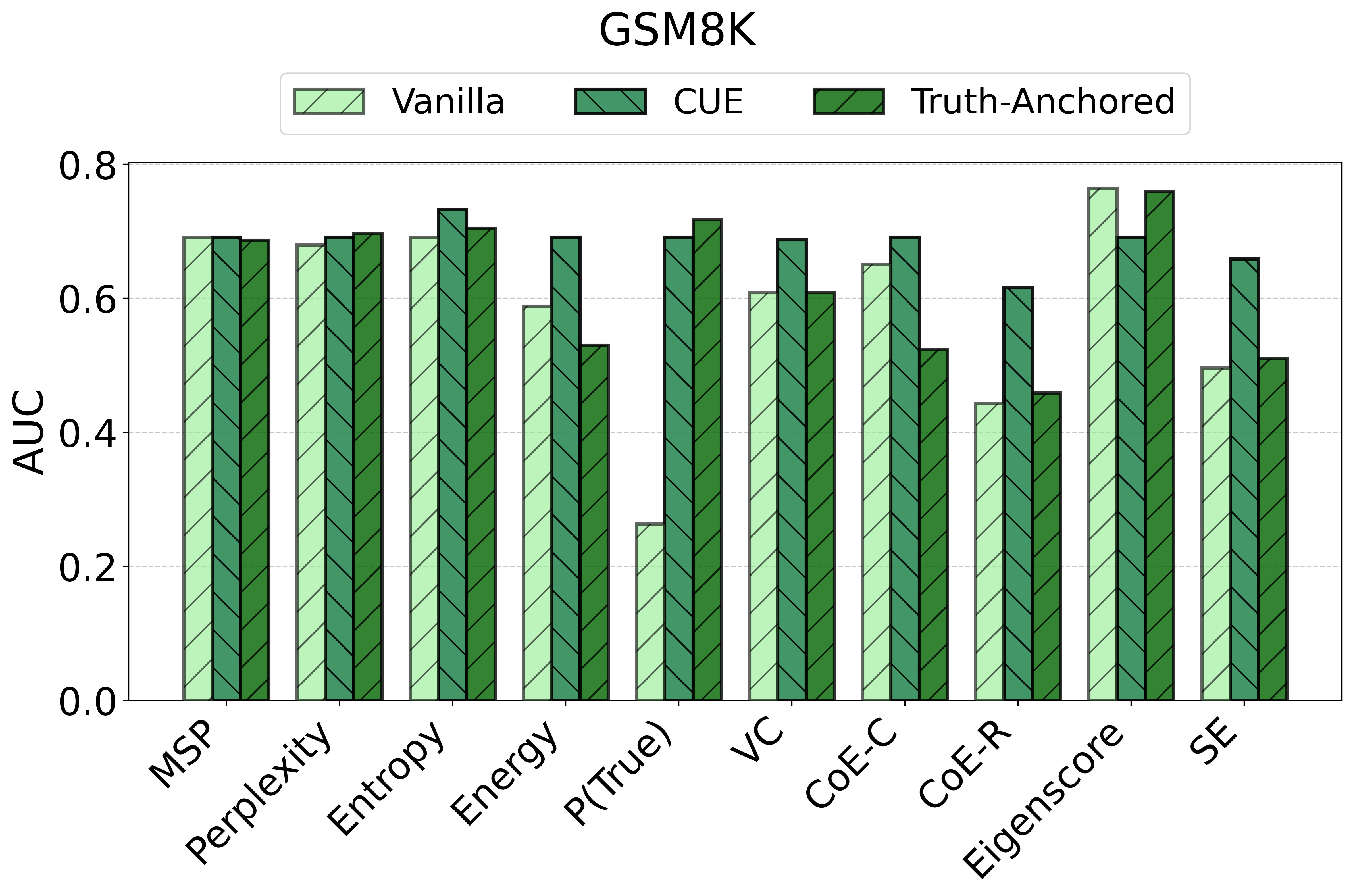}
    \end{subfigure}
    \caption{Performance on GSM8K with Ministral-3-8B.}
    \label{fig:cue_gsm}
\end{figure}

\begin{table}[htbp]
    \centering
    \renewcommand{\arraystretch}{1.2}
    \setlength{\tabcolsep}{1mm}
    \resizebox{\linewidth}{!}{
    \begin{tabular}{l ccc|ccc}
    \toprule
    \multirow{2}{*}{\textbf{Dataset}} & \multicolumn{3}{c}{\textbf{ECE ($\downarrow$)}} & \multicolumn{3}{c}{\textbf{AUC ($\uparrow)$}} \\
    \cmidrule (lr){2-7}
    & \textbf{Vanilla} & \textbf{TAC} & \textbf{$\Delta$} & \textbf{Vanilla} & \textbf{TAC} & \textbf{$\Delta$} \\
    \midrule
    CSQA & 14.16 & 4.03 &  \textcolor{darkgreen}{-10.13} & 56.85 & 63.10 &  \textcolor{darkgreen}{+6.25} \\
    MedMCQA & 23.12 & 5.59 &  \textcolor{darkgreen}{-17.53} & 56.16 & 64.28 &  \textcolor{darkgreen}{+8.12} \\
    \bottomrule
    \end{tabular}
    }
    \caption{Average performance and change on CSQA and MedMCQA with Gemma-2-9B.}
    \label{tab:extra_datasets}
\end{table}

\subsection{Specialised Datasets}
\label{subsec:spec_data}
We have mainly focused on fact-heavy question-answering datasets.
Here, we provide a preliminary analysis of applying our TAC to more specialised datasets, namely the mathematical dataset \textbf{GSM8K} \citep{cobbe2021training} with Ministral-3-8B (\Cref{fig:cue_gsm}), commonsense reasoning task \textbf{CSQA} \citep{talmor-etal-2019-commonsenseqa} and medical-domain dataset \textbf{MedMCQA} \citep{pal2022medmcqa} with  Gemma-2-9B (\Cref{tab:extra_datasets}).
TAC yields consistent improvements, similar to the short-form factual datasets.
Evaluation with tasks beyond QA, such as open-ended generation, remains future work.

\subsection{Inter-score Anchoring}

TAC explicitly aims to improve calibration, but it cannot recover high discriminability if the initial score is uninformative.
To achieve this, additional information is necessary.
Here, we explore this setting, where we assume access to two sources of uncertainty scores, and combine the two scores to inject external information not available within each score alone.
Specifically, we consider distinct raw scores $S_i$ and $S_j$, $i\neq j$, and learn the mapping $m_\theta(S_\mathrm{concat})$ from the concatenated input $S_\mathrm{concat}=[S_i,S_j]$, our pairwise variant of TAC.
We provide the results in \Cref{fig:inter}.

With a combination of two uncertainty scores, stronger discrimination can be recovered.
We hypothesise this is because each score captures a different facet of truth uncertainty, especially for those scores belonging to different families, or formulated using distinct sources of information.
For example, CoE-C and VC scores are derived with different information.
CoE-C utilises internal state information, by tracking the changes between successive layers in their embeddings, whereas VC operates at the output token level. With pairwise TAC, AUC noticeably improves.
However, it remains to be seen how we can optimally combine scores to achieve even better discriminability.

\begin{figure}
    \centering
    \includegraphics[width=\linewidth]{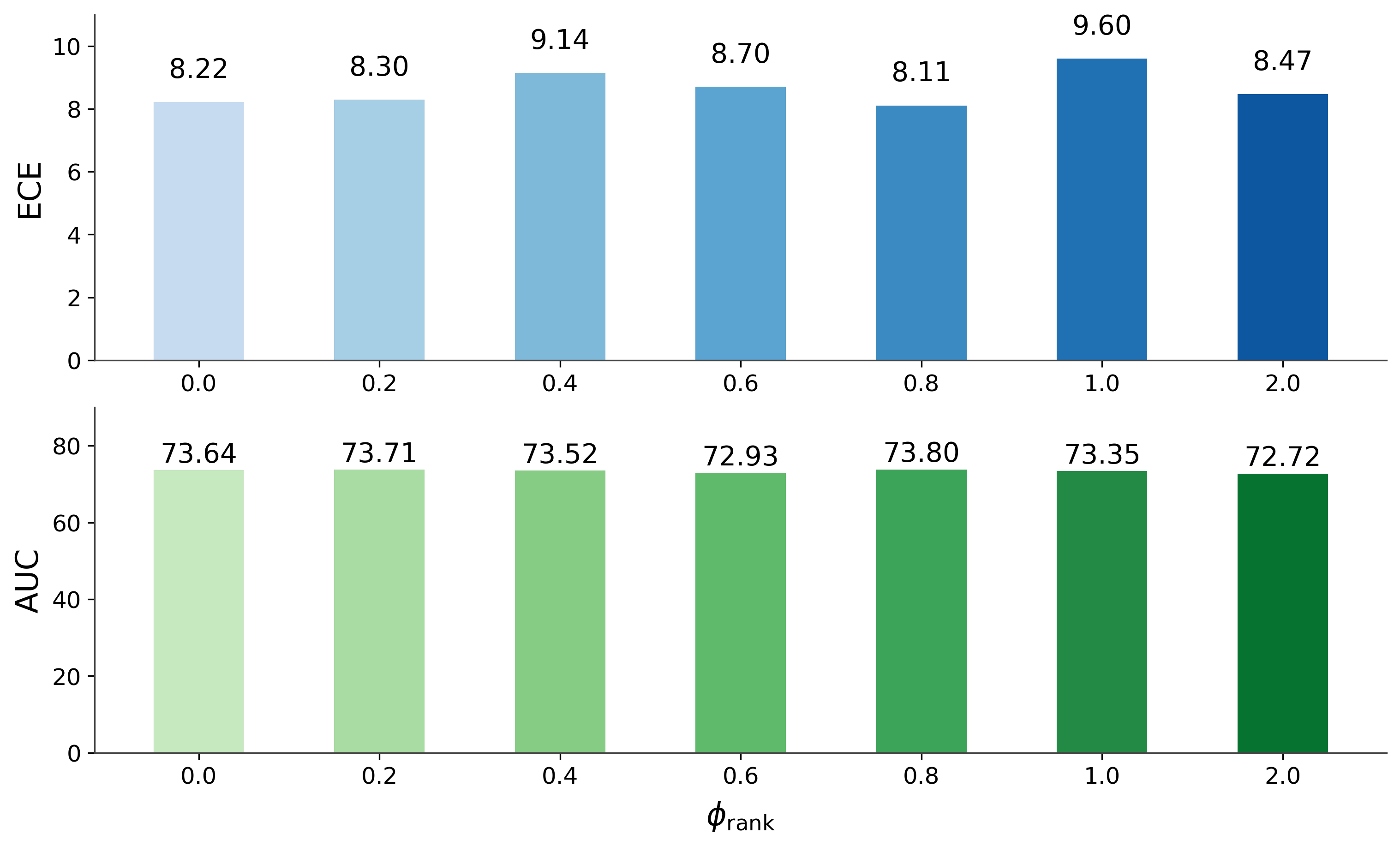}
    \caption{Sensitivity analysis of $\phi_\textrm{rank}$.}
    \label{fig:sensitivity}
\end{figure}

\subsection{Sensitivity Analysis}
\label{subsec:sensitivity}

\Cref{fig:sensitivity} presents the sensitivity analysis of $\phi_\textrm{rank}$. While the ranking loss can lead to gains when $\phi_{\textrm{rank}}>0$, it is not consistently helpful as performance is unstable, especially for ECE.
We believe this happens because the ranking loss optimises relative ordering, while the main BCE objective optimises calibrated correctness probabilities. These goals are related but not identical: pushing pairs apart may help AUC in some cases but can also distort the probability mapping learned by BCE, hurting calibration. In addition, because TAC operates on a scalar score, the auxiliary term may not be able to effectively rank pairs given this weak signal.

\section{Conclusion}

In this work, we formalise proxy failure in LLM uncertainty estimation, explaining the phenomenon where performance of widely used heuristic scores degrades substantially across configurations. We argue that truth anchoring (TAC), our proposed post-hoc calibration method on uncertainty scores, even with imperfect supervision, offers a principled and practical remedy, paving the way for more reliable and meaningful uncertainty estimates.

\section*{Limitations}

Truth anchoring successfully targets calibration, as supported by the significant reduction in ECE. On the other hand, discriminability is not guaranteed to improve, and can sometimes slightly worsen. Inter-score anchoring can be useful, but the optimal combination remains to be analysed.
Furthermore, the scope of our study is limited to fact-heavy QA datasets. Future work may consider other specialised or long-form reasoning-intensive tasks.
Finally, we examined moderately sized LLMs, ranging from 3 to 9 billion parameters. Experiments with larger and more capable models could strengthen our claims.
 
\section*{Ethical Considerations}

Even though our method is shown to improve confidence calibration, by design, it can only provide meaningful uncertainty measures in applications with clear ground truth. Users should avoid deploying Truth Anchoring in open-ended tasks, due to the risk of biased or subjective outputs being represented as confident.

\section*{Acknowledgements}

This research is supported by the RIE2025 Industry Alignment Fund – Industry Collaboration Projects (IAF-ICP) (Award I2301E0026), administered by A*STAR, as well as supported by Alibaba Group and NTU Singapore through Alibaba-NTU Global e-Sustainability CorpLab (ANGEL).
We would like to thank the anonymous reviewers for their valuable feedback.

\bibliography{custom}

\appendix

\section{Proofs}
\label{app:proofs}

In this section, we provide the proofs to \Cref{prop:miAUC} and \Cref{prop:peFail}.

\begin{lemma}
\label{lem:aucTV}

For any score $S$, with $P_+=P(S\mid C=1)$ and $P_-=P(S\mid C=0)$,

\begin{equation}
    \left|\mathrm{AUC}(S) - \frac{1}{2}\right| \leq \delta(P_+, P_-),
\end{equation}
where $\delta$ is the total variation distance.

\end{lemma}

\begin{proof}

Let $S_+ \sim P_+$ and $S_- \sim P_-$ be continuous random variables with PDFs $p_+(s)$ and $p_-(s)$, and CDFs $F_+(s)$ and $F_-(s)$, respectively.
By law of total probability, we can condition AUC on the positive class $S_+=s$
\begin{equation}
\label{eq:auc1}
\begin{split}
    \mathrm{AUC} &= P(S_+>S_-)\\
    &= \int_{-\infty}^{\infty}P(S_-<S_+\mid S_+=s)\, p_+(s) \,ds\\
    &= \int_{-\infty}^{\infty}F_-(s)\, p_+(s) \,ds
\end{split}
\end{equation}

Next, draw $S_+'$ such that $S_+,S_+'\stackrel{\textit{i.i.d}}{\sim} P_+$. Similarly,
\begin{equation}
\label{eq:auc2}
    \frac{1}{2} = P(S_+>S_+') = \int_{-\infty}^{\infty} F_+(s)\,p_+(s)\,ds
\end{equation}

Then, combining \Cref{eq:auc1} and \ref{eq:auc2}, and applying absolute value to both sides and the triangle inequality yields
\begin{equation}
    \left|\mathrm{AUC} - \frac{1}{2} \right|\leq \int_{-\infty}^{\infty} \left|F_-(s) - F_+(s)\right|p_+(s)\,ds
\end{equation}

By definition, the total variation $\delta$ between $P_+$ and $P_-$ is
\begin{equation}
    \delta\,(P_+, P_-) = \sup_{A}\left| P_+(A) -P_-(A) \right|
\end{equation}

Taking $A=(-\infty, s]$, $F_+(s) = P_+(A)$ and $F_-(s) = P_-(A)$, for all $s$, we have
\begin{equation}
    \left| F_+(s) -F_-(s)\right| \leq \delta\,(P_+, P_-)
\end{equation}

Thus,

\begin{equation}
\begin{split}
    \left|\mathrm{AUC} - \frac{1}{2} \right|&\leq \int_{-\infty}^{\infty} \delta\,(P_+,P_-)p_+(s)\,ds\\
    &= \delta\,(P_+,P_-) \int_{-\infty}^{\infty}p_+(s)\,ds\\
    &= \delta\,(P_+,P_-)
\end{split}
\end{equation}
\end{proof}

\miAUC*

\begin{proof}

Let $M=P(S)=p_cP_+ + (1-p_c)P_-$. Since $C$ is binary, mutual information can be decomposed as

\begin{equation}
\label{eq:mi_decomp}
    I(C,S) = p_c \mathrm{KL}(P_+\|M) + (1-p_c)\mathrm{KL}(P_-\|M)
\end{equation}
We first relate $P_+$ and $P_-$ to $M$. Because 

\begin{equation}
\begin{split}
    P_+ - M &= P_+ - (p_cP_+ + (1-p_c) P_-) \\
    &= (1-p_c)(P_+-P_-),
\end{split}
\end{equation}
so
\begin{equation}
    \delta(P_+, M) = (1-p_c)\,\delta(P_+,P_-)
\end{equation}
Similarly, $M-P_- = p_c (P_+ - P_-)$, so
\begin{equation}
    \delta(P_-, M) = p_c\,\delta(P_+,P_-)
\end{equation}
Applying Pinsker's inequality to each term in \Cref{eq:mi_decomp},

\begin{equation}
\begin{split}
    \mathrm{KL}(P_+\|M) \geq&\; 2(1-p_c)^2\,\delta(P_+,P_-)^2,\\
    \mathrm{KL}(P_-\|M) &\geq\; 2p_c^2\,\delta(P_+,P_-)^2
\end{split}
\end{equation}
Substituting into \Cref{eq:mi_decomp} yields 

\begin{equation}
\begin{split}
    I(C,S) \geq &\; 2p_c(1-p_c)^2\,\delta(P_+,P_-)^2 \\
    &+  2(1-p_c)p_c^2\,\delta(P_+,P_-)^2 \\
    = & 2p_c(1-p_c)\delta\,(P_+,P_-)^2
\end{split}
\end{equation}
Hence, 

\begin{equation}
    \delta\,(P_+,P_-) \leq \sqrt{\frac{I(C,S)}{2p_c(1-p_c)}}
\end{equation}
Finally, by \Cref{lem:aucTV},

\begin{equation}
    \left| \mathrm{AUC(S)} - \frac{1}{2} \right| \; \leq \; \sqrt{\frac{I(C,S)}{2p_c(1-p_c)}},
\end{equation}
\end{proof}

\peFail*

\begin{proof}

Let $\varepsilon>0$. Assume for every query $x$, both correct and wrong responses can be generated. For each $x$, choose one correct $y_+(x)$ and one wrong response $y_-(x)$. Let $\tau(x)$ denote the first token position at which the two sequences differ.

\paragraph{Step 1: Construct a base setting with zero information.} Define a conditional distribution $P_0(\cdot\mid x)$ by $P_0(y_+\mid x)=P_0(y_-\mid x)=1/2$.

Under $P_0$, for any $x$, the next-token distribution is deterministic at every step except the first divergence time $\tau(x)$, where it is uniform over two tokens. Hence, along either support trajectory,

\begin{equation}
    H\left(P_0(\cdot\mid x,y_{<t})\right)
    =
    \begin{cases}
    \log 2, & t=\tau(x),\\
    0, & t\neq \tau(x).
    \end{cases}
\end{equation}

The entire entropy profile is identical along $y^{(1)}(x)$ and $y^{(0)}(x)$, and so
\begin{equation}
    S_{pe}(x,y^{(1)}(x))=S_{pe}(x,y^{(0)}(x))=s_0(x).    
\end{equation}

Under $P_0$, $P_+=P_-$, by construction, $I_{P_0}(C,S_{pe})=0$.

\paragraph{Step 2: Mix in with arbitrary component.} Let $P_1(\cdot \mid x)$ be any conditional distribution, for example, the distribution induced by the actual deployed LLM. Let $B\sim\mathrm{Bern}(\lambda)$, where $\lambda\in (0,1)$, and 

\begin{equation}
    P(\cdot\mid x) = (1-\lambda) P_0 + \lambda P_1
\end{equation}

\paragraph{Step 3: Bound mutual information.} Using chain rule,

\begin{equation}
\begin{split}
    I(C;S_{pe}) &\leq I(C;S_{pe},B) \\
    &= I(C;B) + I(C;S_{pe}|B)
\end{split}
\end{equation}
Since $B\sim\mathrm{Bern}(\lambda)$, $I(C,B)\leq H(B) = h(\lambda)$, where $h(\lambda)$ is the binary entropy.
Further,

\begin{equation}
\begin{split}
    I(C;S_{pe}\mid B) & = (1-\lambda) I(C;S_{pe}\mid B=0)\\
    &+ \lambda I(C;S_{pe}\mid B=1)
\end{split}
\end{equation}

Under $B=0$, the distribution is exactly $P_0$ with zero information. Under $B=1$, since $C$ is binary, 
\begin{equation}
    I(C;S_{pe}\mid B=1) \leq H(C\mid B=1) \leq \log2
\end{equation}

Therefore,
\begin{equation}
    I(C;S_{pe}) \leq h(\lambda) + \lambda \log2
\end{equation}

RHS tends to $0$ as $\lambda \to 0$. Hence, for any $\varepsilon$, $\exists\lambda$ such that $h(\lambda) + \lambda \log2\leq \varepsilon$.
This yields $I(C;S_{pe})\leq \varepsilon$, and the AUC bound follows directly from \Cref{prop:miAUC}.

\end{proof}

\section{Supplementary Experiments}
\label{app:add_expt}

\paragraph{Qualitative examples.} \Cref{tab:examples} illustrates instances of the proxy failure phenomenon, where entropy fails to track correctness. We provide four successful examples of TAC calibrating confidence toward the correctness label. For overly confident scores on incorrect answers, TAC pushes confidence down; conversely, for initially low-confidence scores on correct answers, TAC pulls confidence up. Additionally, we highlight a false negative case, where TAC wrongly suppresses confidence on a correct generation.
These examples demonstrate that on aggregate, TAC improves truth alignment but, as a scalar calibration method, cannot guarantee correction of every individual response.

\begin{table}[htbp]
\centering
\setlength{\tabcolsep}{4mm}
\renewcommand{\arraystretch}{1.25}
\resizebox{\linewidth}{!}{
\begin{tabular}{lcc}
\toprule
     \textbf{Dataset} & \textbf{$\Delta$(AUC) (SD)} & \textbf{$\Delta$(ECE) (SD)} \\
     \midrule
     TriviaQA & +3.24 (0.14) & -19.35 (1.16) \\
     SciQ & +5.01 (0.36) & -25.69 (1.37) \\
     PopQA & +2.85 (0.14) & -18.76 (2.77) \\
     \bottomrule
\end{tabular}
}
\caption{Five-seed average changes with Ministral-3-8B.}
\label{tab:stat_sig}
\end{table}

\paragraph{Statistical significance.} To evaluate the statistical significance of our main results in \Cref{tab:main_result}, we additionally run the Ministral-3-8B experiments using five different seeds.
We observe that TAC performance is stable, with small variability across runs.
Further, we perform two-sided one-sample $t$-tests on the average $\Delta$(AUC) values against zero. On TriviaQA, SciQ, PopQA, the $p$-values are $9.6\times10^{-7}$, $6.1\times 10^{-6}$, and $1.2\times 10^{-6}$, respectively. These results suggest that the observed AUC improvements, on average, are statistically significant, and are not artifacts of stochastic MLP training.

\paragraph{More comparison with CUE.} We present more comparison with CUE on SciQ \Cref{fig:cue_sciq} and PopQA in \Cref{fig:cue_popqa}.

\begin{figure}[htbp]
    \centering
    \begin{subfigure}{\linewidth}
        \includegraphics[width=\linewidth]{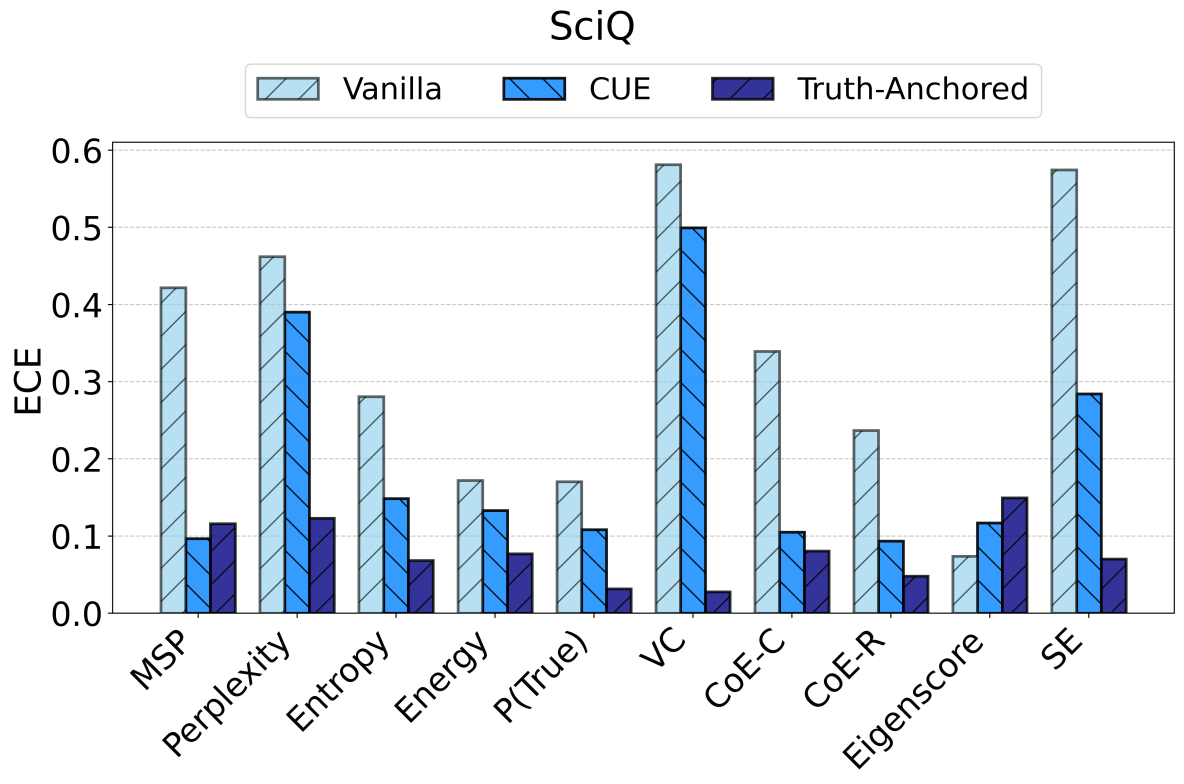}
    \end{subfigure}
    
    \begin{subfigure}{\linewidth}
        \includegraphics[width=\linewidth]{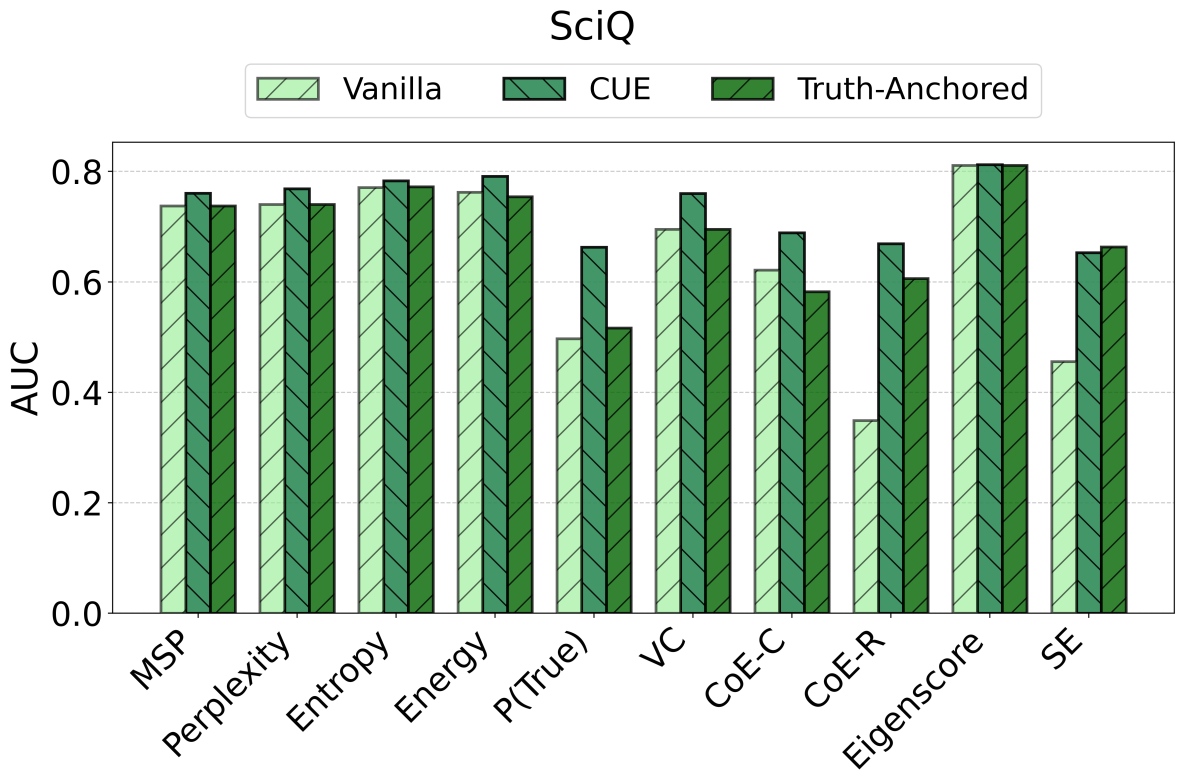}
    \end{subfigure}
    \caption{Performance of vanilla, CUE, and truth-anchored scores.}
    \label{fig:cue_sciq}
\end{figure}

\begin{figure}[!ht]
    \centering
    \begin{subfigure}{\linewidth}
        \includegraphics[width=\linewidth]{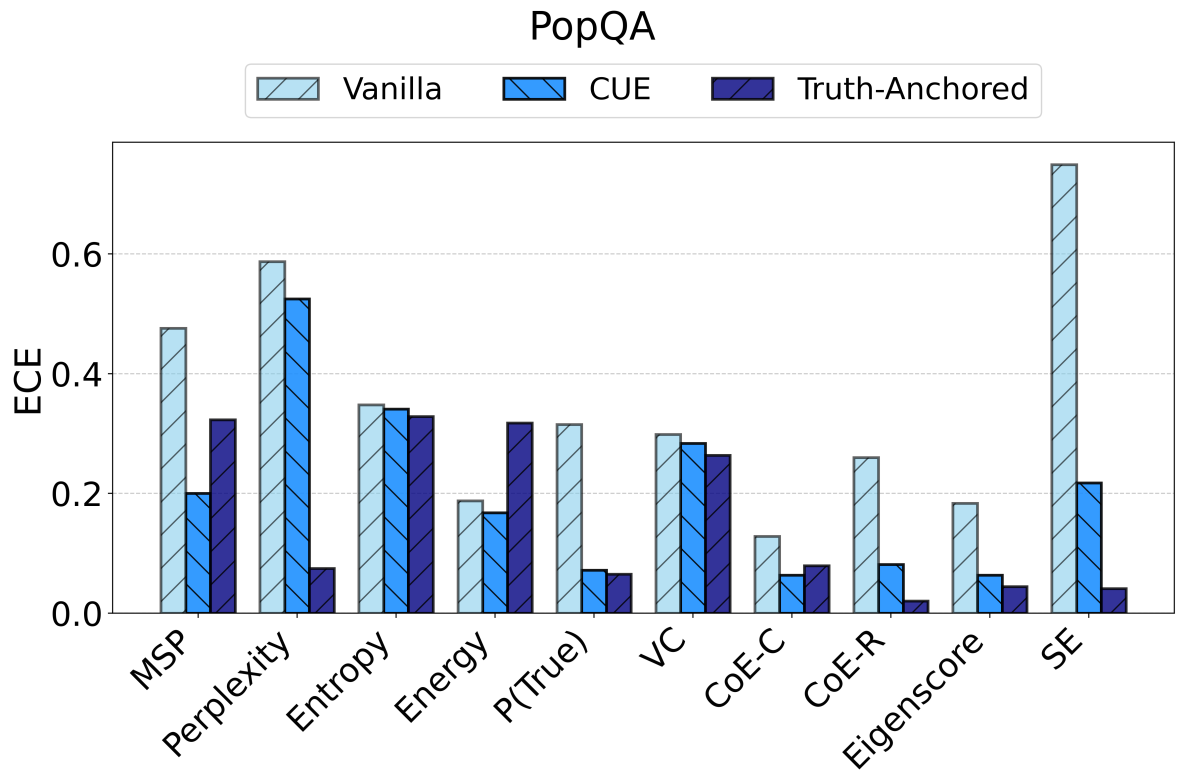}
    \end{subfigure}
    
    \begin{subfigure}{\linewidth}
        \includegraphics[width=\linewidth]{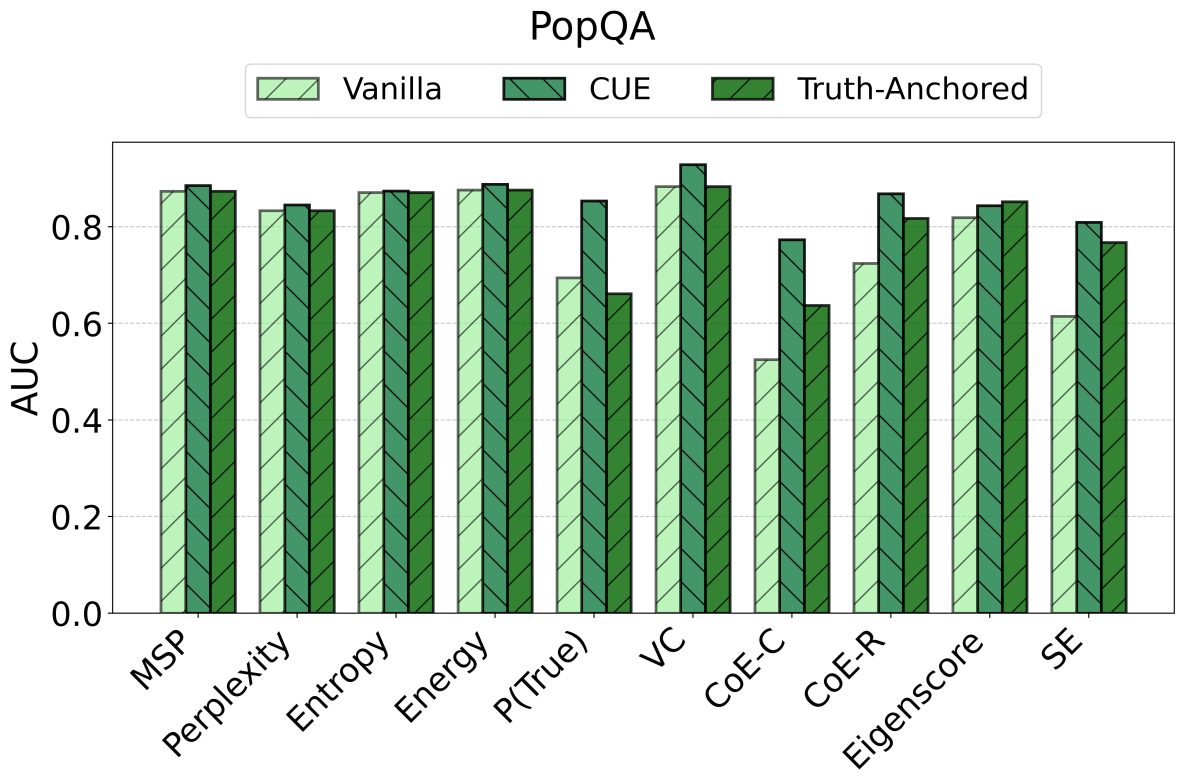}
    \end{subfigure}
    \caption{Performance of vanilla, CUE, and truth-anchored scores.}
    \label{fig:cue_popqa}
\end{figure}

\begin{table*}[!ht]
\centering
\renewcommand{\arraystretch}{1.25}
\setlength{\tabcolsep}{2pt}
\resizebox{\textwidth}{!}{
\begin{tabular}{l ccc|ccc|ccc|ccc|ccc|ccc}
\toprule
& \multicolumn{6}{c}{\textbf{TriviaQA}} & \multicolumn{6}{c}{\textbf{SciQ}} & \multicolumn{6}{c}{\textbf{PopQA}} \\
\cmidrule (lr){2-7}\cmidrule (lr){8-13} \cmidrule (lr){14-19}
\textbf{Method} & \multicolumn{3}{c}{\textbf{ECE ($\downarrow$)}} & \multicolumn{3}{c}{\textbf{AUC ($\uparrow$)}} & \multicolumn{3}{c}{\textbf{ECE ($\downarrow$)}} & \multicolumn{3}{c}{\textbf{AUC ($\uparrow$)}} & \multicolumn{3}{c}{\textbf{ECE ($\downarrow$)}} & \multicolumn{3}{c}{\textbf{AUC ($\uparrow$)}}\\
\cmidrule (lr){2-4}\cmidrule (lr){5-7}\cmidrule (lr){8-10}\cmidrule (lr){11-13} \cmidrule (lr){14-16}\cmidrule (lr){17-19}
& \textbf{Vanilla} & \textbf{TAC} & \textbf{$\Delta$}& \textbf{Vanilla} & \textbf{TAC} & \textbf{$\Delta$} & \textbf{Vanilla} & \textbf{TAC} & \textbf{$\Delta$}& \textbf{Vanilla} & \textbf{TAC} & \textbf{$\Delta$} & \textbf{Vanilla} & \textbf{TAC} & \textbf{$\Delta$}& \textbf{Vanilla} & \textbf{TAC} & \textbf{$\Delta$} \\

\midrule
\headercolor
\multicolumn{19}{c}{\textbf{Llama-3.2-3B}} \\

MSP & 48.25 & 13.41 & \textcolor{darkgreen}{-34.84} & 76.15 & 76.15 & 0.00 & 47.96 & 13.23 & \textcolor{darkgreen}{-34.73} & 75.79 & 75.79 & 0.00 & 70.96 & 2.74 & \textcolor{darkgreen}{-68.22} & 53.46 & 61.51 & \textcolor{darkgreen}{8.05} \\
Perplexity & 58.46 & 6.51 & \textcolor{darkgreen}{-51.95} & 69.19 & 67.11 & \textcolor{darkred}{-2.09} & 50.47 & 8.70 & \textcolor{darkgreen}{-41.77} & 72.88 & 72.25 & \textcolor{darkred}{-0.63} & 80.05 & 2.28 & \textcolor{darkgreen}{-77.76} & 48.25 & 57.41 & \textcolor{darkgreen}{9.16} \\
Entropy & 37.83 & 18.88 & \textcolor{darkgreen}{-18.95} & 84.60 & 84.60 & 0.00 & 35.06 & 15.56 & \textcolor{darkgreen}{-19.49} & 76.28 & 76.28 & 0.00 & 58.48 & 1.71 & \textcolor{darkgreen}{-56.77} & 65.41 & 64.20 & \textcolor{darkred}{-1.21} \\
Energy & 17.07 & 5.39 & \textcolor{darkgreen}{-11.68} & 78.92 & 79.12 & \textcolor{darkgreen}{0.20} & 15.55 & 13.84 & \textcolor{darkgreen}{-1.71} & 74.09 & 74.09 & 0.00 & 35.56 & 5.95 & \textcolor{darkgreen}{-29.61} & 69.93 & 63.01 & \textcolor{darkred}{-6.92} \\
P(True) & 57.55 & 6.15 & \textcolor{darkgreen}{-51.40} & 50.74 & 50.99 & \textcolor{darkgreen}{0.25} & 61.05 & 2.33 & \textcolor{darkgreen}{-58.72} & 61.00 & 61.00 & 0.00 & 71.68 & 25.49 & \textcolor{darkgreen}{-46.19} & 84.13 & 84.13 & 0.00 \\
VC & 17.78 & 20.38 & \textcolor{darkred}{2.60} & 85.56 & 85.56 & 0.00 & 47.73 & 14.02 & \textcolor{darkgreen}{-33.71} & 66.41 & 66.41 & 0.00 & 8.15 & 30.11 & \textcolor{darkred}{21.96} & 94.77 & 94.77 & 0.00 \\
CoE-C & 20.90 & 7.63 & \textcolor{darkgreen}{-13.27} & 59.18 & 73.38 & \textcolor{darkgreen}{14.20} & 17.62 & 7.28 & \textcolor{darkgreen}{-10.33} & 69.79 & 70.36 & \textcolor{darkgreen}{0.57} & 2.50 & 1.61 & \textcolor{darkgreen}{-0.89} & 52.84 & 67.07 & \textcolor{darkgreen}{14.23} \\
CoE-R & 22.71 & 4.07 & \textcolor{darkgreen}{-18.64} & 74.43 & 75.30 & \textcolor{darkgreen}{0.87} & 40.04 & 4.99 & \textcolor{darkgreen}{-35.05} & 59.86 & 60.77 & \textcolor{darkgreen}{0.91} & 31.08 & 3.53 & \textcolor{darkgreen}{-27.54} & 65.55 & 80.35 & \textcolor{darkgreen}{14.80} \\
Eigenscore & 27.70 & 12.15 & \textcolor{darkgreen}{-15.56} & 68.37 & 78.05 & \textcolor{darkgreen}{9.67} & 20.57 & 7.20 & \textcolor{darkgreen}{-13.37} & 69.81 & 66.94 & \textcolor{darkred}{-2.87} & 65.55 & 3.37 & \textcolor{darkgreen}{-62.18} & 40.44 & 58.29 & \textcolor{darkgreen}{17.85} \\
Semantic Entropy & 59.45 & 6.58 & \textcolor{darkgreen}{-52.86} & 52.08 & 65.69 & \textcolor{darkgreen}{13.61} & 66.03 & 6.16 & \textcolor{darkgreen}{-59.87} & 43.86 & 64.07 & \textcolor{darkgreen}{20.21} & 78.07 & 3.23 & \textcolor{darkgreen}{-74.84} & 49.36 & 55.30 & \textcolor{darkgreen}{5.94} \\

\midrule
\headercolor
\multicolumn{19}{c}{\textbf{Llama-3.1-8B}} \\

MSP & 30.31 & 8.53 & \textcolor{darkgreen}{-21.78} & 75.79 & 75.79 & 0.00 & 46.84 & 6.49 & \textcolor{darkgreen}{-40.35} & 81.93 & 82.39 & \textcolor{darkgreen}{0.46} & 47.85 & 5.61 & \textcolor{darkgreen}{-42.24} & 75.42 & 77.36 & \textcolor{darkgreen}{1.94} \\
Perplexity & 29.34 & 7.96 & \textcolor{darkgreen}{-21.37} & 73.87 & 73.87 & 0.00 & 44.29 & 9.91 & \textcolor{darkgreen}{-34.38} & 81.88 & 81.57 & \textcolor{darkred}{-0.31} & 58.16 & 13.63 & \textcolor{darkgreen}{-44.53} & 71.12 & 71.12 & 0.00 \\
Entropy & 20.53 & 11.66 & \textcolor{darkgreen}{-8.87} & 76.91 & 76.91 & 0.00 & 33.65 & 6.38 & \textcolor{darkgreen}{-27.26} & 82.31 & 82.43 & \textcolor{darkgreen}{0.12} & 34.24 & 22.18 & \textcolor{darkgreen}{-12.06} & 77.40 & 77.40 & 0.00 \\
Energy & 10.21 & 8.52 & \textcolor{darkgreen}{-1.69} & 70.53 & 70.53 & 0.00 & 17.64 & 4.28 & \textcolor{darkgreen}{-13.36} & 67.54 & 65.14 & \textcolor{darkred}{-2.40} & 10.57 & 6.91 & \textcolor{darkgreen}{-3.66} & 76.97 & 76.56 & \textcolor{darkred}{-0.41} \\
P(True) & 38.56 & 4.12 & \textcolor{darkgreen}{-34.44} & 62.16 & 62.21 & \textcolor{darkgreen}{0.05} & 53.99 & 4.18 & \textcolor{darkgreen}{-49.81} & 62.86 & 65.90 & \textcolor{darkgreen}{3.04} & 57.01 & 17.04 & \textcolor{darkgreen}{-39.96} & 80.66 & 80.66 & 0.00 \\
VC & 38.98 & 2.16 & \textcolor{darkgreen}{-36.82} & 62.87 & 63.62 & \textcolor{darkgreen}{0.75} & 57.62 & 4.20 & \textcolor{darkgreen}{-53.42} & 67.74 & 68.67 & \textcolor{darkgreen}{0.93} & 35.24 & 17.15 & \textcolor{darkgreen}{-18.09} & 75.97 & 76.15 & \textcolor{darkgreen}{0.18} \\
CoE-C & 46.47 & 8.89 & \textcolor{darkgreen}{-37.59} & 62.43 & 63.08 & \textcolor{darkgreen}{0.64} & 31.04 & 6.20 & \textcolor{darkgreen}{-24.84} & 66.22 & 66.79 & \textcolor{darkgreen}{0.58} & 19.90 & 3.53 & \textcolor{darkgreen}{-16.37} & 43.10 & 58.30 & \textcolor{darkgreen}{15.20} \\
CoE-R & 16.40 & 3.92 & \textcolor{darkgreen}{-12.48} & 54.60 & 65.19 & \textcolor{darkgreen}{10.60} & 20.65 & 4.19 & \textcolor{darkgreen}{-16.46} & 38.01 & 63.98 & \textcolor{darkgreen}{25.97} & 19.24 & 6.31 & \textcolor{darkgreen}{-12.93} & 68.03 & 75.92 & \textcolor{darkgreen}{7.89} \\
Eigenscore & 9.63 & 4.75 & \textcolor{darkgreen}{-4.88} & 76.42 & 76.76 & \textcolor{darkgreen}{0.34} & 12.89 & 9.81 & \textcolor{darkgreen}{-3.08} & 83.16 & 82.69 & \textcolor{darkred}{-0.47} & 21.09 & 3.60 & \textcolor{darkgreen}{-17.49} & 69.30 & 75.57 & \textcolor{darkgreen}{6.28} \\
SE & 40.39 & 3.85 & \textcolor{darkgreen}{-36.54} & 42.88 & 68.68 & \textcolor{darkgreen}{25.80} & 58.44 & 3.65 & \textcolor{darkgreen}{-54.78} & 49.34 & 74.63 & \textcolor{darkgreen}{25.29} & 68.49 & 3.21 & \textcolor{darkgreen}{-65.28} & 53.25 & 60.28 & \textcolor{darkgreen}{7.03} \\

\bottomrule
\end{tabular}
}

\caption{Additional results with Llama models.}

\label{tab:add_results}
\end{table*}

\begin{table*}[!t]
\centering
\renewcommand{\arraystretch}{1.25}
\resizebox{\textwidth}{!}{
\begin{tabular}{p{11cm} c >{\centering\arraybackslash}p{3cm} >{\centering\arraybackslash}p{3cm}}
\toprule
\textbf{Question} & \textbf{LLM Answer} & \textbf{Label} & \textbf{Vanilla $\rightarrow$ TAC} \newline \textbf{Confidence} \\
\midrule
Which businessman on his elevation to the House of Lords, by Gordon Brown, in 2009, took the title Baron of Clapton? & Lord Alan Sugar & Correct & \textcolor{darkgreen}{0.172 $\rightarrow$ 0.271} \\ \hline
Which 19th century British novelist, who worked as a Surveyor for the Post Office, was reputedly instrumental in the introduction of the pillar box in 1852? & Charles Dickens & Incorrect \newline (Anthony Trollope) & \textcolor{darkgreen}{0.877 $\rightarrow$ 0.591} \\ \hline
Approximately 20\% of the atmosphere is made of which element? & Nitrogen & Incorrect \newline (Oxygen) & \textcolor{darkgreen}{0.890 $\rightarrow$ 0.594} \\ \hline
What model describing atomic structure is often referred to as a planetary model? & Rutherford model & Incorrect \newline (Bohr model) & \textcolor{darkgreen}{0.692 $\rightarrow$ 0.267} \\ \hline
What is another word for cellular drinking? & Pinocytosis & Correct & \textcolor{darkred}{0.632 $\rightarrow$ 0.416 \newline (false negative)} \\
\bottomrule

\end{tabular}
}
\caption{Qualitative examples of proxy failure and subsequent TAC calibration. Vanilla confidence is derived from entropy. \textcolor{darkgreen}{Green} indicates movement toward the correctness label, with 1 representing correct generations.}
\label{tab:examples}
\end{table*}

\paragraph{More LLMs.} \Cref{tab:add_results} shows the performance with Llama models \citep{grattafiori2024llama}.

\paragraph{Noisy and few-shot supervision.} \Cref{fig:ds,fig:cr} plot the individual TAC scores for each few-shot or noisy setting for TriviaQA.

\begin{figure}
    \centering
    \includegraphics[width=\linewidth]{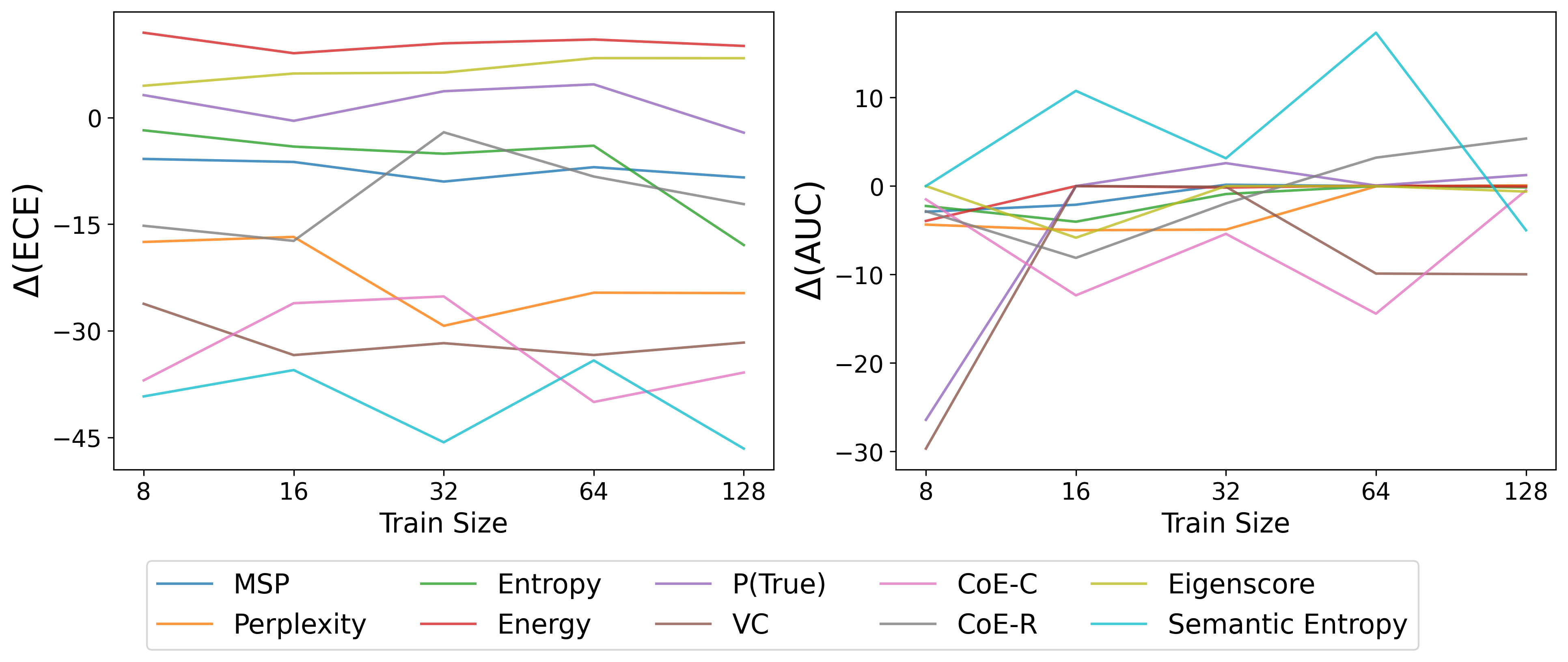}
    \caption{$\Delta$(ECE) and $\Delta$(AUC) on TriviaQA with few-shot supervision.}
    \label{fig:ds}
\end{figure}

\begin{figure}
    \centering
    \includegraphics[width=\linewidth]{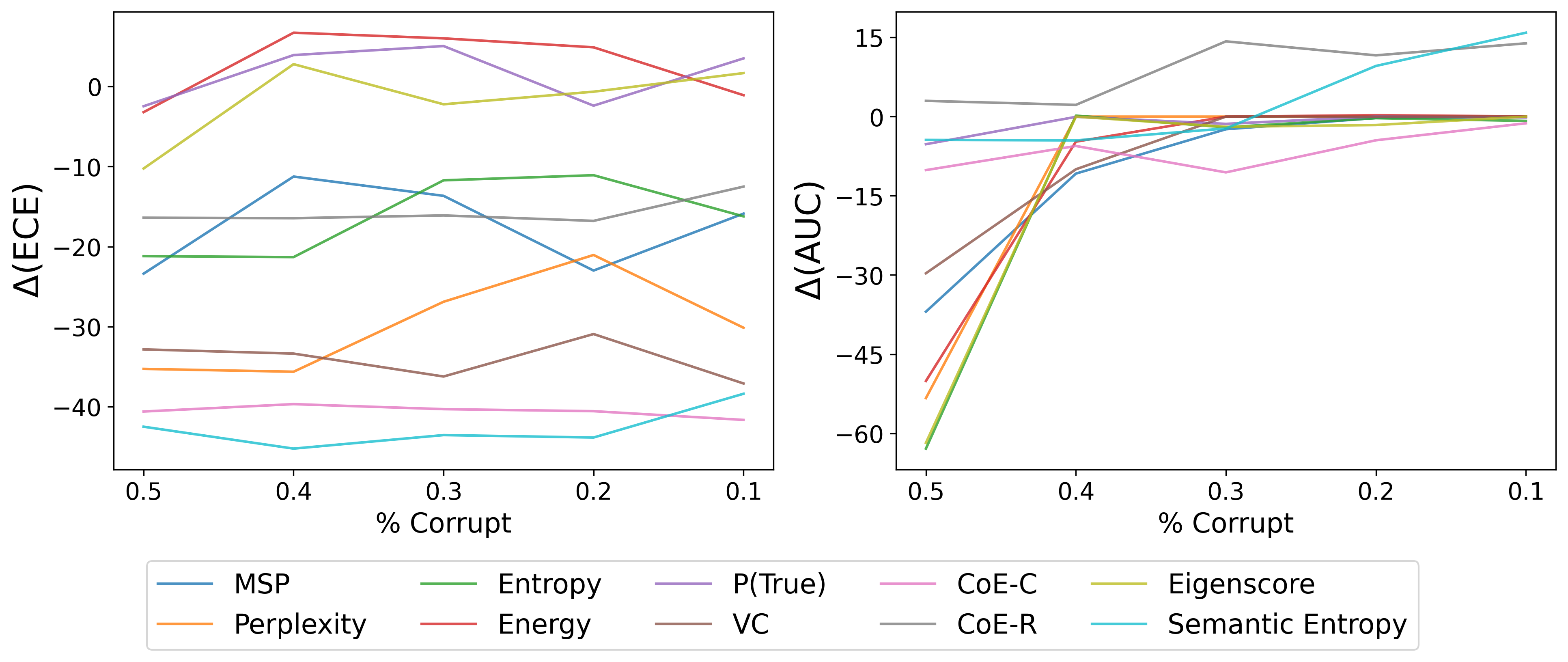}
    \caption{$\Delta$(ECE) and $\Delta$(AUC) on TriviaQA with noisy supervision.}
    \label{fig:cr}
\end{figure}

\end{document}